\documentclass[11pt]{article}

\newcommand{\citet}[1]{\cite{#1}}
\newcommand{\citep}[1]{\cite{#1}}

\usepackage{times,xspace}
\usepackage{setup-alex}
\usepackage{setup-crowds} 

\usepackage{fullpage}

\usepackage[suppress]{color-edits}
\addauthor{om}{green}   
\addauthor{it}{magenta} 
\addauthor{vk}{blue}    
\addauthor{as}{red}     

\usepackage{pgfplots}
\pgfplotsset{
	tick label style={font=\tiny},
	label style={font=\footnotesize},
	legend style={font=\footnotesize},
}

\usepackage{graphicx,caption,subcaption}

\renewcommand{\eqref}[1]{Equation~(\ref{#1})}
\newcommand{\fakeItem}[1][$\bullet$]{\vspace{2mm}{\bf #1}~~}

\begin{document}

\title{Adaptive Crowdsourcing Algorithms for the Bandit Survey Problem}

\author{Ittai Abraham%
    \thanks{Microsoft Research Silicon Valley,
    Mountain View, CA 94040, USA.
    Email: \{ittaia,slivkins\}@microsoft.com.}
\and Omar Alonso%
    \thanks{Microsoft Corporation,
    Mountain View, CA 94040, USA.
    Email: \{omalonso, vakandyl\}@microsoft.com.}
\and Vasilis Kandylas$^\ddag$
\and Aleksandrs Slivkins$^\dag$
}

\date{First version: February 2013 \\ This version: May 2013}

\maketitle

\begin{abstract}
Very recently crowdsourcing has become the de facto platform for distributing and collecting human computation for a wide range of tasks and applications such as information retrieval, natural language processing and machine learning. Current crowdsourcing platforms have some limitations in the area of quality control. Most of the effort to ensure good quality has to be done by the experimenter who has to manage the number of workers needed to reach good results.

We propose a simple model for adaptive quality control in crowdsourced multiple-choice tasks which we call the \emph{bandit survey problem}. This model is related to, but technically different from the well-known multi-armed bandit problem. We present several algorithms for this problem, and support them with analysis and simulations.
Our approach is based in our experience conducting relevance evaluation for a large commercial search engine.
\end{abstract}

\section{Introduction}
\label{sec:intro}

In recent years there has been a surge of interest in automated methods for \emph{crowdsourcing}: a distributed model for problem-solving and experimentation that involves broadcasting the problem or parts thereof to multiple independent, relatively inexpensive workers and aggregating their solutions. Automation and optimization of this process at a large scale allows to significantly reduce the costs associated with setting up, running, and analyzing the experiments. Crowdsourcing is finding applications across a wide range of domains in information retrieval, natural language processing and machine learning.

A typical crowdsourcing workload is partitioned into \emph{microtasks} (also called Human Intelligence Tasks), where each microtask has a specific, simple structure and involves only a small amount of work. Each worker is presented with multiple microtasks of the same type, to save time on training. The rigidity and simplicity of the microtasks' structure ensures consistency across multiple multitasks and across multiple workers.

An important industrial application of crowdsourcing concerns web search. One specific goal in this domain is  \emph{relevance assessment}: assessing the relevance of search results. One popular task design involves presenting a microtask in the form of a query along with the results from the search engine. The worker has to answer one question about the relevance of the query to the result set. Experiments such as these are used to evaluate the performance of a search engine, construct training sets, and discover queries which require more attention and potential algorithmic tuning.

\xhdr{Stopping / selection issues.}
The most basic experimental design issue for crowdsourcing is the \emph{stopping issue}: determining how many workers the platform should use for a given microtask before it stops and outputs the aggregate answer. The workers in a crowdsourcing environment are not very reliable, so multiple workers are usually needed to ensure a sufficient confidence level. There is an obvious tradeoff here: using more workers  naturally increases the confidence of the aggregate result but it also increases the cost and time associated with the experiment. One fairly common heuristic is to use less workers if the microtasks seem easy, and more workers if the microtasks seem hard. However, finding a sweet-spot may be challenging, especially if different microtasks have different degrees of difficulty.

Whenever one can distinguish between workers, we have a more nuanced \emph{selection issue}: which workers to choose for a given microtask? The workers typically come from a large, loosely managed population. Accordingly, the skill levels vary over the population, and are often hard to predict in advance. Further, the relative skill levels among workers may depend significantly on a particular microtask or type of microtasks. Despite this uncertainty, it is essential to choose workers that are suitable or cost-efficient for the micro-task at hand, to the degree of granularity allowed by the crowdsourcing platform. For example, while targeting individual workers may be infeasible, one may be able to select some of the workers' attributes such as age range, gender, country, or education level. Also, the crowdsourcing platform may give access to multiple third-party providers of workers, and allow to select among those.

\xhdr{Our focus.}
This paper is concerned with a combination of the stopping / selection issues discussed above. We seek a clean setting so as to understand these issues at a more fundamental level.

\OMIT{As an initial paper
which studies these issues in the context of crowdsourcing,}

We focus on the scenario where several different populations of workers are available and can be targeted by the algorithm. As explained above, these populations may correspond to different selections of workers' attributes, or to multiple available third-party providers. We will refer to such populations as \emph{crowds}. We assume that the quality of each crowd depends on a particular microtask, and is not known in advance.

Each microtask is processed by an online algorithm which can adaptively decide which crowd to ask next. Informally, the goal is target the crowds that are most suitable for this microtask. Eventually the algorithm must stop and output the aggregate answer.

This paper focuses on processing a single microtask. This allows us to simplify the setting: we do not need to model how the latent quantities are correlated across different microtasks, and how the decisions and feedbacks for different microtasks are interleaved over time. Further, we separate the issue of learning the latent quality of a crowd for a given microtask from the issue of learning the (different but correlated) quality parameters of this crowd across multiple microtasks.

\xhdr{Our model: the \problem.}
We consider microtasks that are multiple-choice questions: one is given a set $\options$ of possible answers, henceforth called \emph{options}.
\asedit{We allow more than two options. (In fact, we find this case to be much more difficult than the case of only two options.)}
Informally, the microtask has a unique correct answer $x^* \in \options$, and the high-level goal of the algorithm is to find it.

The algorithm has access to several crowds: populations of workers. Each crowd $i$ is represented by a distribution $\D_i$ over $\options$, called the \emph{response distribution} for $i$. We assume that all crowds agree on the correct answer:%
\footnote{Otherwise the algorithm's high-level goal is less clear. We chose to avoid this complication in the current version.}
some option $x^*\in \options$ is the unique most probable option for each $\D_i$.

In each round $t$, the algorithm picks some crowd $i=i_t$ and receives an independent sample from the corresponding response distribution $\D_i$. Eventually the algorithm must stop and output its guess for $x^*$. Each crowd $i$ has a known per-round cost $c_i$. The algorithm has two objectives to minimize: the total cost $\sum_t c_{i_t}$ and the \emph{error rate}: the probability that it makes a mistake, i.e. outputs an option other than $x^*$.
\asedit{There are several ways to trade off these two objectives; we discuss this issue in more detail later in this section.}

The independent sample in the above model abstracts the following interaction between the algorithm and the platform: the platform supplies a worker from the chosen crowd, the algorithm presents the microtask to this worker, and the worker picks some option.

{\em Alternative interpretation.}
The crowds can correspond not to different populations of workers but to different ways of presenting the same microtask. For example, one could vary the instructions, the order in which the options are presented, the fonts and the styles, and the accompanying images.

{\em The name of the game.} Our model is similar to the extensively studied \emph{multi-armed bandit problem} (henceforth, \emph{MAB}) in that in each round an algorithm selects one alternative from a fixed and known set of available alternatives, and the feedback depends on the chosen alternative. However, while an MAB algorithm collects rewards, an algorithm in our model collects a \emph{survey} of workers' opinions. Hence we name our model the {\bf \problem}.

\xhdr{Discussion of the model.}
The \problem belongs to a broad class of online decision problems with explore-exploit tradeoff: that is, the algorithm faces a tradeoff between collecting information (\emph{exploration}) and taking advantage of the information gathered so far (\emph{exploitation}). The
paradigmatic problem in this class is MAB: in each round an algorithm picks one alternative (\emph{arm}) from a given set of arms, and receives a randomized, time-dependent reward associated with this arm; the goal is to maximize the total reward over time. Most papers on explore-exploit tradeoff concern MAB and its variants.

The \problem is different from MAB in several key respects. First, the feedback is different: the feedback in MAB is the reward for the chosen alternative, whereas in our setting the feedback is the opinion of a worker from the chosen crowd. While the information received by a \Algorithm can be interpreted as a ``reward", the value of such reward is not revealed to the algorithm and moreover not explicitly defined. Second, the algorithm's goal is different: the goal in MAB is to maximize the total reward over time, whereas the goal in our setting is to output the correct answer. Third, in our setting there are two types of ``alternatives": crowds and options in the microtask. Apart from repeatedly selecting between the crowds, a \Algorithm needs to output one option: the aggregate answer for the microtask.

An interesting feature of the \problem is that an algorithm for this problem consists of two components: a \emph{crowd-selection algorithm} -- an online algorithm that decides which crowd to ask next, and a \emph{stopping rule} which decides whether to stop in a given round and which option to output as the aggregate answer. These two components are, to a large extent, independent from one another: as long as they do not explicitly communicate with one another (or otherwise share a common communication protocol) any crowd-selection algorithm can be used in conjunction with any stopping rule.%
\footnote{The no-communication choice is quite reasonable: in fact, it can be complicated to design a reasonable \Algorithm that requires explicit communication between the crowd-selection algorithm and a stopping rule.}
\asedit{The conceptual separation of a \Algorithm into the two components is akin to one in Mechanism Design, where it is very useful to separate a ``mechanism" into an ``allocation algorithm" and a ``payment rule", even though the two components are not entirely independent of one another.}

\xhdr{\asedit{Trading off the total cost and the error rate.}}
\asedit{In the \problem, an algorithm needs to trade off the two objectives: the total cost and the error rate.} In a typical application, the customer is willing to tolerate a certain error rate, and wishes to minimize the total cost as long as the error rate is below this threshold. However, as the error rate depends on the problem instance, there are several ways to make this formal. Indeed, one could consider the worst-case error rate (the maximum over all problem instances), a typical error rate (the expectation over a given ``typical" distribution over problem instance), or a more nuanced notion such as the maximum over a given family of ``typical" distributions. Note that the ``worst-case" guarantees may be overly pessimistic, whereas considering ``typical" distributions makes sense only if one knows what these distributions are.

\asedit{For our theoretical guarantees, we focus on the worst-case error rate, and use the \emph{bi-criteria objective}, a standard approach from theoretical computer science literature: we allow some slack on one objective, and compare on another. In our case, we allow slack on
the worst-case error rate, and compare on the expected total cost. More precisely: we consider a benchmark with some worst-case error rate $\delta>0$ and optimal total cost given this $\delta$, allow our algorithm to have worst-case error rate which is (slightly) larger than $\delta$, and compare its expected total cost to that of the benchmark.}

\asedit{Moreover, we obtain provable guarantees in terms of a different, problem-specific objective: use the same stopping rule, compare on the expected total cost. We believe that such results are well-motivated by the structure of the problem, and provide a more informative way to compare crowd-selection algorithms.}

In our experiments, we fix the per-instance error rate, and compare on the expected total cost.

An alternative objective is to assign a monetary penalty to a mistake, and optimize the overall cost, i.e. the cost of labor minus the penalty. However, it may be exceedingly difficult for a customer to assign such monetary penalty,%
\footnote{In particular, this was the case in the authors' collaboration with a commercial crowdsourcing platform.}
whereas it is typically feasible to specify tolerable error rates. \asedit{While we think this alternative is worth studying, we chose not to follow it in this paper.}

\xhdr{Our approach: independent design.}
Our approach is to design crowd-selection algorithms and stopping rules independently from one another. We make this design choice in order to make the overall algorithm design task more tractable. While this is not the only possible design choice, we find it productive, as it leads to a solid theoretical framework and algorithms that are practical and theoretically founded.

Given this ``independent design'' approach, one needs to define the design goals for each of the two components. These goals are not immediately obvious. Indeed, two stopping rules may compare differently depending on the problem instance and the crowd-selection algorithms they are used with. Likewise, two crowd-selection algorithms may compare differently depending on the problem instance and the stopping rules they are used with. Therefore the notions of optimal stopping rule and optimal crowd-selection algorithm are not immediately well-defined.

We resolve this conundrum as follows. We design crowd-selection algorithms that work well across a wide range of stopping rules. For a fair comparison between crowd-selection algorithms, we use them with the \emph{same} stopping rule (see Section~\ref{sec:benchmarks} for details), and argue that such comparison is consistent across different stopping rules.

\xhdr{Our contributions.} We introduce the \problem and present initial results in several directions: benchmarks, algorithms, theoretical analysis, and experiments.

We are mainly concerned with the design of crowd-selection algorithms. Our crowd-selection algorithms work with arbitrary stopping rules. While we provide a specific (and quite reasonable) family of stopping rules for concreteness, third-party stopping rules can be easily plugged in.

For the theoretical analysis of crowd-selection algorithms, we use a standard benchmark: the best time-invariant policy given all the latent information. The literature on online decision problems typically studies a deterministic version of this benchmark: the best fixed alternative (in our case, the best fixed crowd). We call it the \emph{deterministic benchmark}. We also consider a randomized version, whereby an alternative (crowd) is selected independently from the same distribution in each round; we call it the \emph{randomized benchmark}. The technical definition of the benchmarks, as discussed in Section~\ref{sec:benchmarks}, roughly corresponds to equalizing the worst-case error rates and comparing costs.

The specific contributions are as follows.

\fakeItem[(1)] We largely solve the \problem as far as the deterministic benchmark is concerned. We design two crowd-selection algorithms, obtain strong provable guarantees, and show that they perform well in experiments.

\asedit{Our provable guarantees are as follows. If our crowd-selection algorithm uses the same stopping rule as the benchmark, we match the expected total cost of the deterministic benchmark up to a small additive factor, assuming that all crowds have the same per-round costs. This result holds, essentially, for an arbitrary stopping rule. We obtain a similar, but slightly weaker result if crowds can have different per-round costs. Moreover, we can restate this as a bi-criteria result, in which we incur a small additive increase in the expected total cost and $(1+k)$ multiplicative increase in the worst-case error rate, where $k$ is the number of crowds. The contribution in these results is mostly conceptual rather than technical: it involves ``independent design" as discussed above, and a ``virtual rewards" technique which allows us to take advantage of the MAB machinery.}

For comparison, we consider a naive crowd-selection algorithm that tries each crowd in a round-robin fashion. We prove that this algorithm, and more generally any crowd-selection algorithm that does not adapt to the observed workers' responses, performs very badly against the deterministic benchmark. While one expects this on an intuitive level, the corresponding mathematical statement is not easy to prove. In experiments, our proposed crowd-selection algorithms perform much better than the naive approach.

\fakeItem[(2)] We observe that the randomized benchmark dramatically outperforms the deterministic benchmark on some problem instances.
This is a very unusual property for an online decision problem.%
\footnote{We are aware of only one published example of an online decision problem with this property, in a very different context of dynamic pricing~\cite{DynPricing-ec12}. However, the results in~\cite{DynPricing-ec12} focus on a special case where the two benchmarks essentially coincide.}
(However, the two benchmarks coincide when there are only two possible answers.)

We design an algorithm which significantly improves over the expected total cost of the deterministic benchmark on some problem instances (while not quite reaching the randomized benchmark), \asedit{when both our algorithm and the benchmarks are run with the same stopping rule}. This appears to be the first published result in the literature on online decision problems where an algorithm provably improves over the deterministic benchmark.

\asedit{We can aslo restate this result in terms of the bi-criteria objective. Then we suffer a $(1+k)$ multiplicative increase in the worst-case error rate.}

\fakeItem[(3)] We provide a specific stopping rule for concreteness; this stopping rule is simple, tunable, has nearly optimal theoretical guarantees (in a certain formal sense), and works well in experiments.

\OMIT{This paper is mainly concerned with the design of crowd-selection algorithms. In particular, we do not attempt to fully optimize the stopping rules.  We provide a specific stopping rule for concreteness; this stopping rule is simple, tunable, has nearly optimal theoretical guarantees (in a certain formal sense), and works well in experiments. However, a third-party stopping rule can be easily plugged in.}

\OMIT{
We view crowdsourcing as a tool that a  human computation system can use to distribute tasks (work). Similar to the terminology presented in
\cite{Law11}, we define human computation as a computation (or task) that is performed by a human.}

\xhdr{Preliminaries and notation.}
There are $k$ crowds and $n$ options (possible answers to the microtask). $\options$ denotes the set of all options. An important special case is \emph{uniform costs}: all $c_i$ are equal; then the total cost is simply the stopping time.

Fix round $t$ in the execution of a bandit survey algorithm. Let $\NSamples{i}{t}$ be the number of rounds before $t$ in which crowd $i$ has been chosen by the algorithm. Among these rounds, let $\NSamples{i}{t}(x)$ be the number of times a given option $x\in\options$ has been chosen by this crowd. The \emph{empirical distribution} $\empir{\D}{i}{t}$ for crowd $i$ is given by
    $\empir{\D}{i}{t}(x) = \NSamples{i}{t}(x)/\NSamples{i}{t}$
for each option $x$. We use $\empir{\D}{i}{t}$ to approximate the (latent) response distribution $\D_i$.

Define the \emph{gap} $\gap(\D)$ of a finite-support probability distribution $\D$ as the difference between the largest and the second-largest probability values in $\D$. If there are only two options ($n=2$),
the gap of a distribution over $\options$ is simply the bias towards the correct answer. Let
    $\gap_i = \gap(\D_i)$
and
    $\empir{\gap}{i}{t} = \gap(\empir{\D}{i}{t})$
be, respectively, the \emph{gap} and the \emph{empirical gap} of crowd $i$.

We will use vector notation over crowds: the \emph{cost vector}
    $\vec{c} = (c_1 \LDOTS c_k)$,
the \emph{gap vector}
    $\vec{\gap} = (\gap_1 \LDOTS \gap_k)$,
and the \emph{response vector}
    $\vec{D}(x) = (\D_1(x) \LDOTS \D_k(x))$
for each option $x\in \options$.

\xhdr{Map of the paper.} The rest of the paper is organized as follows. As a warm-up and a foundation, we consider stopping rules for a single crowd (Section~\ref{sec:single-crowd}). Benchmarks are formally defined in Section~\ref{sec:benchmarks}. Design of crowd-selection algorithms with respect to the deterministic benchmark is treated in Section~\ref{sec:multi-crowd}. We discuss the randomized benchmark in Section~\ref{sec:randomized-benchmark}; we design and analyze an algorithm for this benchmark in Section~\ref{sec:randomized-benchmark-algs}. Results with respect to the bi-criteria benchmark are in Section~\ref{sec-bicriteria}.
We present our experimental results Section~\ref{sec:expts-single-crowd} and Section~\ref{sec:expts-multi-crowds}, respectively for a single crowd and for selection over multiple crowds. We discuss related work in Section~\ref{sec:related-work}, and open questions in Section~\ref{sec:questions}.

\OMIT{To improve the flow of the paper, some of the proofs and some of the plots are moved to the Appendix.}

\section{A warm-up: single-crowd stopping rules}
\label{sec:single-crowd}

Consider a special case with only one crowd to choose from. It is clear that whenever a \Algorithm decides to stop, it should output the most frequent option in the sample. Therefore the algorithm reduces to what we call a \emph{single-crowd stopping rule}: an online algorithm which in every round inputs an option $x\in\options$ and decides whether to stop. When multiple crowds are available, a single-crowd stopping rule can be applied to each crowd separately. This discussion of the single-crowd stopping rules, together with the notation and tools that we introduce along the way, forms a foundation for the rest of the paper.

A single-crowd stopping rule is characterized by two quantities that are to be minimized: the expected stopping time and the \emph{error rate}: the probability that once the rule decides to stop, the most frequent option in the sample is not $x^*$. Note that both quantities depend on the problem instance; therefore we leave the bi-criteria objective somewhat informal at this point.

\xhdr{A simple single-crowd stopping rule.}
We suggest the following single-crowd stopping rule:
\begin{align}\label{eq:stopping-rule}
\text{Stop if}\;
\empir{\gap}{i}{t}\, \NSamples{i}{t} > \errorC\,\sqrt{\NSamples{i}{t}}.
\end{align}
Here $i$ is the crowd the stopping rule is applied to, and $\errorC$ is the \emph{quality parameter} which indirectly controls the tradeoff between the error rate and the expected stopping time. Specifically, increasing $\errorC$ decreases the error rate and increases the expected stopping time. If there are only two options, call them $x$ and $y$, then the left-hand side of the stopping rule is simply
    $|\NSamples{i}{t}(x) - \NSamples{i}{t}(y)|$.

The right-hand side of the stopping rule is a confidence term, which should be large enough to guarantee the desired confidence level. The $\sqrt{\NSamples{i}{t}}$ is there because the standard deviation of the Binomial distribution with $N$ samples is proportional to $\sqrt{N}$.

In our experiments, we use a ``smooth" version of this stopping rule: we randomly round the confidence term to one of the two nearest integers. In particular, the smooth version is meaningful even with $\errorC <1$
(whereas the deterministic version with $\errorC <1$ always stops after one round).

\xhdr{Analysis.} We argue that the proposed single-crowd stopping rule is quite reasonable. To this end, we obtain a provable guarantee on the tradeoff between the expected stopping time and the worst-case error rate. Further, we prove that this guarantee is nearly optimal across all single-crowd stopping rules. Both results above are in terms of the gap of the crowd that the stopping rule interacts with. We conclude that the gap is a crucial parameter for the \problem.

\begin{theorem}\label{thm:single-crowd}
Consider the stopping rule~\refeq{eq:stopping-rule} with
    $\errorC =  \sqrt{\log (\tfrac{n}{\delta}\, \NSamples{i}{t}^2 )}$,
for some $\delta>0$.
The error rate of this stopping rule is at most $O(\delta)$, and the expected stopping time is at most
    $O\left( \gap_i^{-2}\, \log \tfrac{n}{\delta \gap_i} \right)$.
\end{theorem}

The proof of Theorem~\ref{thm:single-crowd}, and several other proofs in the paper, rely on the Azuma-Hoeffding inequality. More specifically, we use the following corollary: for each $C>0$, each round $t$, and each option $x\in \options$
\begin{align}\label{eq:conf-rad}
\Pr\left[\;
    |\D_i(x)-\empir{\D}{i}{t}(x)| \leq C/\sqrt{\NSamples{i}{t}}
\;\right]
    \geq 1-e^{-\Omega(C^2)}.
\end{align}
In particular, taking the Union Bound over all options $x\in\options$, we obtain:
\begin{align}\label{eq:conf-rad-gap}
\Pr\left[\;
    |\empir{\gap}{i}{t} - \gap_i| \leq C/\sqrt{\NSamples{i}{t}}
\;\right]
    \geq 1- n\,e^{-\Omega(C^2)}.
\end{align}

\begin{proof}[Proof of Theorem~\ref{thm:single-crowd}]
Fix $a\geq 1$ and let $C_t =  \sqrt{\log (a\, \tfrac{n}{\delta}\, \NSamples{i}{t}^2 )}$. Let $\mathcal{E}_{x,t}$ be the event in ~\eqref{eq:conf-rad} with $C = C_t$. Consider the event that $\mathcal{E}_{x,t}$ holds for all options $x\in \options$ and all rounds $t$; call it the \emph{clean event}. Taking the Union Bound, we see that the clean event holds with probability at least $1-O(\delta/a)$.

First, assuming the clean event
we have
    $|\gap_i - \empir{\gap}{i}{t}| \leq 2\, C_t/\sqrt{\NSamples{i}{t}} $
for all rounds $t$. Then the stopping rule~\refeq{eq:stopping-rule} stops as soon as
    $\gap_i \geq 3\, C_t/\sqrt{\NSamples{i}{t}}$,
which happens as soon as
    $\NSamples{i}{t} = O\left( \gap_i^{-2}\, \log \tfrac{an}{\delta \gap_i} \right)$.
Integrating this over all $a\geq 1$, we derive that the expected stopping time is as claimed.

Second, take $a=1$ and assume the clean event. Suppose the stopping rule stops at some round $t$. Let $x$ be the most probable option after this round. Then
    $\empir{\D}{i}{t}(x)-\empir{\D}{i}{t}(y) \geq C_t/\sqrt{\NSamples{i}{t}}$
for all options $y\neq x$. It follows that
    $D_i(x) > D_i(y)$
for all options $y\neq x$, i.e. $x$ is the correct answer.
\end{proof}

The following lower bound easily follows from classical results on coin-tossing. Essentially, one needs at least $\Omega(\gap^{-2})$ samples from a crowd with gap $\gap>0$ to obtain the correct answer.

\begin{theorem}\label{thm:single-crowd-LB}
Let $R_0$ be any single-crowd stopping rule with worst-case error rate less than $\delta$.
When applied to a crowd with gap $\gap>0$, the expected stopping time of $R_0$ is at least
    $\Omega(\gap^{-2} \log \tfrac{1}{\delta})$.
\end{theorem}

While the upper bound in Theorem~\ref{thm:single-crowd} is close to the lower bound in Theorem~\ref{thm:single-crowd-LB}, it is possible that one can obtain a more efficient version of Theorem~\ref{thm:single-crowd} using more sophisticated versions of Azuma-Hoeffding inequality such as, for example, the Empirical Bernstein Inequality.


\xhdr{Stopping rules for multiple crowds.}
For multiple crowds, we consider stopping rules that are composed of multiple instances of a given single-crowd stopping rule $R_0$; we call them \emph{composite} stopping rules. Specifically, we have one instance of $R_0$ for each crowd (which only inputs answers from this crowd), and an additional instance of $R_0$ for the \emph{total crowd} -- the entire population of workers. The composite stopping rule stops as soon as some $R_0$ instances stops, and outputs the majority option for this instance.%
\footnote{If $R_0$ is randomized, then each instance of $R_0$ uses an independent random seed. If multiple instances of $R_0$ stop at the same time, the aggregate answer is chosen uniformly at random among the majority options for the stopped instances.}
Given a crowd-selection algorithm $\A$, let $\cost(\A|R_0)$ denote its expected total cost (for a given problem instance) when run together with the composite stopping rule based on $R_0$.

\section{Omniscient benchmarks for crowd selection}
\label{sec:benchmarks}

We consider two ``omniscient" benchmarks for crowd-selection algorithms: informally, the best fixed crowd $i^*$ and the best fixed distribution $\mu^*$ over crowds, where $i^*$ and $\mu^*$ are chosen given the latent information: the response distributions of the crowds.
Both benchmarks treat all their inputs as a single data source, and are used in conjunction with a given single-crowd stopping rule $R_0$ (and hence depend on the $R_0$).

\xhdr{Deterministic benchmark.}
Let $\cost(i|R_0)$ be the expected total cost of always choosing crowd $i$, with $R_0$ as the stopping rule. We define the \emph{deterministic benchmark} as the crowd $i^*$ that minimizes $\cost(i|R_0)$ for a given problem instance. In view of the analysis in Section~\ref{sec:single-crowd}, our intuition is that $\cost(i|R_0)$ is approximated by $c_i/\gap_i^2$ up to a constant factor (where the factor may depend on $R_0$ but not on the response distribution of the crowd). The exact identity of the best crowd may depend on $R_0$. For the basic special case of uniform  costs and two options (assuming that the expected stopping time of $R_0$ is non-increasing in the gap), the best crowd is the crowd with the largest gap. In general, we approximate the best crowd by $\argmin_i c_i/\gap_i^2$.

\xhdr{Randomized benchmark.}
Given a distribution $\mu$ over crowds, let $\cost(\mu|R_0)$ be the expected total cost of a crowd-selection algorithm that in each round chooses a crowd independently from $\mu$, treats all inputs as a single data source -- essentially, a single crowd -- and uses $R_0$ as a stopping rule on this data source. The \emph{randomized benchmark} is defined as the $\mu$ that minimizes $\cost(\mu|R_0)$ for a given problem instance. This benchmark is further discussed in Section~\ref{sec:randomized-benchmark}.

\xhdr{Comparison against the benchmarks.} In the analysis, we compare a given crowd-selection algorithm $\A$ against these benchmarks as follows: we use $\A$ in conjunction with the composite stopping rule based on $R_0$, and compare the expected total cost $\cost(\A|R_0)$ against those of the benchmarks.%
\footnote{Using the same $R_0$ roughly equalizes the worst-case error rate between $\A$ and the benchmarks; see Section~\ref{sec-bicriteria} for details.}

Moreover, we derive corollaries with respect to the bi-criteria objective, where the benchmarks choose both the best crowd (resp., best distribution over crowds) and the stopping rule. These corollaries are further discussed in Section~\ref{sec-bicriteria}.

\OMIT{
Let us argue that using the same $R_0$ roughly equalizes the worst-case error rate between $\A$ and the benchmarks. Let $\rho$ be the worst-case error rate of $R_0$, and assume it is achieved for gap $\gap$. Then the worst-case error rate for both benchmarks is $\rho$; it is achieved on a problem instance in which all crowds have gap $\gap$. It is easy to see that the worst-case error rate of $\A$ is at most $(k+1)\rho$, where $k$ is the number of crowds.

We also conjecture that the worst-case error rate for $\A$ is at least $\rho$. In the Appendix (Lemma~\ref{lm:error-rate-LB}), we prove a slightly weaker result: essentially, if the composite stopping rule does not use the total crowd, then the worst-case error rate for $\A$ is at least $\rho\,(1-2k\rho)$.
} 

\section{Crowd selection against the deterministic benchmark}
\label{sec:multi-crowd}

In this section we design crowd-selection algorithms that compete with the deterministic benchmark.

Throughout the section, let $R_0$ be a fixed single-parameter stopping rule. Recall that the deterministic benchmark is defined as $\min \cost(i|R_0)$, where the minimum is over all crowds $i$. We consider arbitrary composite stopping rules based on $R_0$, under a mild assumption that the  $R_0$ does not favor one option over another. Formally, we assume that the probability that $R_0$ stops at any given round, conditional on any fixed history (sequence of observations that $R_0$ inputs before this round), does not change if the options are permuted. Then $R_0$ and the corresponding composite stopping rule are called \emph{symmetric}. For the case of two options (when the expected stopping time of $R_0$ depends only on the gap of the crowd that $R_0$ interacts with) we sometimes make another mild assumption: that the expected stopping time decreases in the gap; we call such $R_0$ \emph{gap-decreasing}.

\subsection{Crowd-selection algorithms}

\xhdr{Virtual reward heuristic.} Our crowd-selection algorithms are based on the following idea, which we call the virtual reward heuristic.%
\footnote{We thank anonymous reviewers for pointing out that our index-based algorithm can be interpreted via virtual rewards. }
Given an instance of the \problem, consider an MAB instance where crowds correspond to arms, and selecting each crowd $i$ results in reward  $f_i = f(c_i/\gap_i^2)$, for some fixed decreasing function $f$.
(Given the discussion in Section~\ref{sec:single-crowd}, we use $c_i/\gap_i^2$ as an approximation for $\cost(i|R_0)$; we can also plug in a better approximation  when and if one is available.)
Call $f_i$ the \emph{virtual reward}; note that it is not directly observed by a \Algorithm, since it depends on the gap $\gap_i$. However, various off-the-shelf bandit algorithms can be restated in terms of the estimated rewards, rather than the actual observed rewards. The idea is to use such bandit algorithms and plug in our own estimates for the rewards.

A bandit algorithm thus applied would implicitly minimize the number of times suboptimal crowds are chosen. This is a desirable by-product of the design goal in MAB, which is to maximize the total (virtual) reward. We are not directly interested in this design goal, but we take advantage of the by-product.

\xhdr{Algorithm 1: \UCB with virtual rewards.} Our first crowd-selection algorithm is based on $\UCB$~\cite{bandits-ucb1}, a standard MAB algorithm. We use virtual rewards
    $f_i = \gap_i/\sqrt{c_i}$.

We observe that \UCB has a property that at each time $t$, it only requires an estimate of $f_i$ and a confidence term for this estimate. Motivated by~\eqref{eq:conf-rad-gap}, we use
    $\empir{\gap}{i}{t}/\sqrt{c_i}$
as the estimate for $f_i$, and
    $C/\sqrt{c_i\, \NSamples{i}{t}}$
as the confidence term. The resulting crowd-selection algorithm, which we call $\AlgUCB$, proceeds as follows. In each round $t$ it chooses the crowd $i$ which maximizes the \emph{index} $I_{i,t}$, defined as
\begin{align}\label{eq:UCB-index}
I_{i,t} = c_i^{-1/2}\left(\,
    \empir{\gap}{i}{t} + C/\sqrt{\NSamples{i}{t}}
\,\right).
\end{align}

\noindent For the analysis, we use~\refeq{eq:UCB-index} with $C = \sqrt{8\log t}$. In our experiments, $C=1$ appears to perform best.

\xhdr{Algorithm 2: Thompson heuristic.}
Our second crowd-selection algorithm, called $\AlgThompson$, is an adaptation of \emph{Thompson heuristic} \cite{Thompson-1933} for MAB to virtual rewards $f_i = \gap_i/\sqrt{c_i}$. The algorithm proceeds as follows. For each round $t$ and each crowd $i$, let $\mathcal{P}_{i,t}$ be the Bayesian posterior distribution for gap $\gap_i$ given the observations from crowd $i$ up to round $t$ (starting from the uniform prior). Sample $\zeta_i$ independently from $\mathcal{P}_{i,t}$. Pick the crowd with the largest \emph{index} $\zeta_i/\sqrt{c_i}$. As in \UCB, the index of crowd $i$ is chosen from the confidence interval for the (virtual) reward of this crowd, but here it is a random sample from this interval, whereas in \UCB it is the upper bound.

It appears difficult to compute the posteriors $\mathcal{P}_{i,t}$ exactly, so in practice an approximation can be used. In our simulations we focus on the case of two options, call them $x,y$. For each crowd $i$ and round $t$, we approximate $\mathcal{P}_{i,t}$ by the Beta distribution with shape parameters
    $\alpha = 1+\NSamples{i}{t}(x)$
and
    $\beta = 1+\NSamples{i}{t}(y)$,
where
    $\NSamples{i}{t}(x)\geq \NSamples{i}{t}(y)$.
(Essentially, we ignore the possibility that $x$ is not the right answer.)

It is not clear how the posterior $\mathcal{P}_{i,t}$ in our problem corresponds to the one in the original MAB problem, so we cannot directly invoke the analyses of Thompson heuristic for MAB~\cite{Thompson-nips11,Thompson-analysis-arxiv11}.

\newcommand{\TP}[1]{T_{\mathtt{Ph#1}}} 

\xhdr{Straw-man approaches.}
We compare the two algorithms presented above to an obvious naive approach: iterate through each crowd in a round-robin fashion. More precisely, we consider a slightly more refined version where in each round the crowd is sampled from a fixed distribution $\mu$ over crowds.
We will call such algorithms \emph{non-adaptive}. The most reasonable version, called \AlgRR (short for ``randomized round-robin'') is to sample each crowd $i$ with probability $\mu_i\sim 1/c_i$.%
\footnote{For uniform  costs it is natural to use a uniform distribution for $\mu$. For non-uniform  costs our choice is motivated by Theorem~\ref{thm:LB}, where it (approximately) minimizes the competitive ratio.}

In the literature on MAB, more sophisticated algorithms are often compared to the basic approach: first explore, then exploit. In our context this means to first \emph{explore} until we can identify the best crowd, then pick this crowd and \emph{exploit}. So for the sake of comparison we also develop a crowd-selection algorithm that is directly based on this approach. (This algorithm is not based on the virtual rewards.) In our experiments we find it vastly inferior to $\AlgUCB$ and $\AlgThompson$.

The ``explore, then exploit" design does not quite work as is: selecting the best crowd with high probability seems to require a high-probability guarantee that this crowd can produce the correct answer with the current data, in which case there is no need for a further exploitation phase (and so we are essentially back to $\AlgRR$). Instead, our algorithm explores until it can identify the best crowd with \emph{low} confidence, then it exploits with this crowd until it sufficiently boosts the confidence or until it realizes that it has selected a wrong crowd to exploit. The latter possibility necessitates a third phase, called \emph{rollback}, in which the algorithm explores until it finds the right answer with high confidence.

The algorithm assumes that the single-crowd stopping rule $R_0$ has a quality parameter $\errorC$ which controls the trade-off between the error rate and the expected running time (as in Section~\ref{sec:single-crowd}). In the exploration phase, we also use a \emph{low-confidence} version of $R_0$ that is parameterized with a lower value $\errorC'<\errorC$; we run one low-confidence instance of $R_0$ for each crowd.

The algorithm, called $\AlgEER$, proceeds in three phases (and stops whenever the composite stopping rule decides so). In the exploration phase, it runs \AlgRR until the low-confidence version of $R_0$ stops for some crowd $i^*$. In the exploitation phase, it always chooses crowd $i^*$. This phase lasts $\alpha$ times as long as the exploration phase, where
the parameter $\alpha$ is chosen so that crowd $i^*$ produces a high-confidence answer w.h.p. if it is indeed the best crowd.%
\footnote{We conjecture that for $R_0$ from Section~\ref{sec:single-crowd} one can take
    $\alpha = \Theta(\errorC/\errorC')$.}
Finally, in the roll-back phase it runs \AlgRR.

\subsection{Analysis: upper bounds}

We start with a lemma that captures the intuition behind the virtual reward heuristic, explaining how it helps to minimize the selection of suboptimal crowds. Then we derive an upper bound for $\AlgUCB$.

\begin{lemma}\label{lm:Ni}
Let
    $i^* = \argmin_i c_i/\eps^2_i$
be the approximate best crowd. Let $R_0$ be a symmetric single-crowd stopping rule.
Then for any crowd-selection algorithm $\A$, letting $N_i$ be \#times crowd $i$ is chosen, we have
$$ \cost(\A|R_0) \leq \cost(i^*|R_0) +
    \textstyle \sum_{i\neq i^*} c_i\, \E[N_i].
$$
\end{lemma}

This is a non-trivial statement because $\cost(i^*|R_0)$ refers not to the execution of $\A$, but to a different execution in which crowd $i^*$ is always chosen. The proof uses a ``coupling argument''.

\begin{proof}
Let $\A^*$ be the crowd-selection algorithm which corresponds to always choosing crowd $i^*$.

To compare $\cost(\A|R_0)$ and $\cost(\A^*|R_0)$, let us assume w.l.o.g. that the two algorithms are run on correlated sources of randomness. Specifically, assume that both algorithms are run on the same realization of answers for crowd $i^*$: the $\ell$-th time they ask this crowd, both algorithms get the same answer. Moreover, assume that the instance of $R_0$ that works with crowd $i^*$ uses the same random seed for both algorithms.

Let $N$ be the realized stopping time for $\A^*$. Then $\A$ must stop after crowd $i^*$ is chosen $N$ times. It follows that the difference in the realized total costs between $\A$ and $\A^*$ is at most
    $\sum_i\; c_i N_i$.
The claim follows by taking expectation over the randomness in the crowds and in the stopping rule.
\end{proof}

\begin{theorem}[\AlgUCB]\label{thm:UCB}
Let
    $i^* = \argmin_i c_i/\gap^2_i$
be the approximate best crowd. Let $R_0$ be a symmetric single-crowd stopping rule. Assume $R_0$  must stop after at most $T$ rounds. Define \AlgUCB by~\refeq{eq:UCB-index} with $C = \sqrt{8\log t}$, for each round $t$. Let
    $\Lambda_i = ( c_i (f_{i^*} - f_i))^{-2}$
and
    $\Lambda = \sum_{i\neq i^*} \Lambda_i$.
Then
$$ \cost(\AlgUCB|R_0) \leq \cost(i^*|R_0) + O(\Lambda \log T).
$$
\end{theorem}

\begin{proof}[Proof Sketch]
Plugging $C = \sqrt{8\log t}$ into~\eqref{eq:conf-rad-gap} and dividing by $\sqrt{c_i}$, we obtain the confidence bound for
    $|f_i - \empir{\gap}{i}{t}/\sqrt{c_i}|$
that is needed in the the original analysis of \UCB in~\cite{bandits-ucb1}. Then, as per that analysis, it follows that for each crowd $i\neq i^*$ and each round $t$ we have
    $ \E[\NSamples{i}{t}] \leq \Lambda_i \log t$.
(This is also not difficult to derive directly.) To complete the proof, note that $t\leq T$ and invoke Lemma~\ref{lm:Ni}.
\end{proof}

Note that the approximate best crowd $i^*$ may be different from the (actual) best arm, so the guarantee in Theorem~\ref{thm:UCB} is only as good as the difference
    $ \cost(i^*|R_0) - \argmin_i \cost(i|R_0)$.
Note that $i^*$ is in fact the best crowd for the basic special case of uniform  costs and two options (assuming that $R_0$ is gap-decreasing).

It is not clear whether the constants $\Lambda_i$ can be significantly improved. For uniform  costs we have
    $\Lambda_i = (\eps_{i^*}- \eps_i)^{-2} $,
which is essentially the best one could hope for. This is because one needs to try each crowd $i\neq i^*$ at least
    $\Omega(\Lambda_i)$
times to tell it apart from crowd $i^*$.
\footnote{This can be proved using an easy reduction from an instance of the MAB problem where each arm $i$ brings reward $1$ with probability $(1+\gap_i)/2$, and reward $0$ otherwise. Treat this as an instance of the \problem, where arms correspond to crowds, and options to rewards. An algorithm that finds the crowd with a larger gap in less than $\Omega(\Lambda_i)$ steps would also find an arm with a larger expected reward, which would violate the corresponding lower bound for the MAB problem (see~\cite{bandits-exp3}).}

\subsection{Analysis: lower bound for non-adaptive crowd selection}

We purpose of this section is argue that non-adaptive crowd-selection algorithms performs badly compared to $\AlgUCB$. We prove that the competitive ratio of any non-adaptive crowd-selection algorithm is bounded from below by (essentially) the number of crowds. We contrast this with an upper bound on the competitive ratio of $\AlgUCB$, which we derive from Theorem~\ref{thm:UCB}.

Here the competitive ratio of algorithm $\A$ (with respect to the deterministic benchmark) is defined as
    $ \max \frac{\cost(\A|R_0)}{\max_i \cost(i|R_0)}$,
where the outer $\max$ is over all problem instances in a given family of problem instances. We focus on a very simple family: problem instances with two options and uniform  costs, in which one crowd has gap $\gap>0$ and all other crowds have gap $0$; we call such  instances \emph{$\gap$-simple}.

Our result holds for a version of a composite stopping rule that does not use the total crowd. Note that considering the total crowd does not, intuitively, make sense for the $\gap$-simple problem instances, and we did not use it in the proof of Theorem~\ref{thm:UCB}, either.

\begin{theorem}[$\AlgRR$]\label{thm:LB}
Let $R_0$ be a symmetric single-crowd stopping rule with worst-case error rate $\rho$. Assume that the composite stopping rule does not use the total crowd. Consider a non-adaptive crowd-selection algorithm $\A$ whose distribution over crowds is $\mu$. Then for each $\gap>0$, the competitive ratio over the $\gap$-simple problem instances is at least
$ \frac{\sum_i c_i\,\mu_i}{\min_i c_i\,\mu_i}\; (1-2k\rho)$,
where $k$ is the number of crowds.
\end{theorem}

Note that
    $\min \, \frac{\sum_i c_i\,\mu_i}{\min_i c_i\,\mu_i} = k$,
where the $\min$ is taken over all distributions $\mu$.
The minimizing $\mu$ satisfies
    $\mu_i \sim 1/c_i$
for each crowd $i$, i.e. if $\mu$ corresponds to \AlgRR. The $(1-2k\rho)$ factor could be an artifact of our somewhat crude method to bound the ``contribution'' of the gap-$0$ crowds. We conjecture that this factor is unnecessary (perhaps under some minor assumptions on $R_0$).

To prove Theorem~\ref{thm:LB}, we essentially need to compare the stopping time of the composite stopping rule $R$ with the stopping time of the instance of $R_0$ that works with the gap-$\gap$ crowd. The main technical difficulty is to show that the other crowds are not likely to force $R$ to stop before this $R_0$ instance does. To this end, we use a lemma that $R_0$ is not likely to stop in finite time when applied to a gap-$0$ crowd.

\begin{lemma}\label{lm:LB-infty}
Consider a symmetric single-crowd stopping rule $R_0$ with worst-case error rate $\rho$. Suppose $R_0$ is applied to a crowd with gap $0$. Then
    $\Pr[\text{$R_0$ stops in finite time}] \leq 2\rho$.
\end{lemma}

\begin{proof}
Intuitively, if $R_0$ stops early if the gap is $0$ then it is likely to make a mistake if the gap is very small but positive. However, connecting the probability in question with the error rate of $R_0$ requires some work.

Suppose $R_0$ is applied to a crowd with gap $\gap$. Let $q(\gap,t,x)$ be the probability that $R_0$ stops at round $t$ and ``outputs'' option $x$ (in the sense that by the time $R_0$ stops, $x$ is the majority vote).

We claim that for all rounds $t$ and each option $x$ we have
\begin{align}\label{eq:lm-LB-infty}
     \lim_{\gap\to 0}\; q(\gap,t,x) = q(0,t,x).
\end{align}
Indeed, suppose not. Then for some $\delta>0$ there exist arbitrarily small gaps $\gap>0$ such that
    $|q(\gap,t,x)-q(0,t,x)| >\delta$.
Thus it is possible to tell apart a crowd with gap $0$ from a crowd with gap $\gap$  by observing $\Theta(\delta^{-2})$ independent runs of $R_0$, where each run continues for $t$ steps. In other words, it is possible to tell apart a fair coin from a gap-$\gap$ coin using $\Theta(t\,\delta^{-2})$ ``coin tosses'', for fixed $t$ and $\delta>0$ and an arbitrarily small $\gap$. Contradiction. Claim proved.

Let $x$ and $y$ be the two options, and let $x$ be the correct answer. Let $q(\gap,t)$ be the probability that $R_0$ stops at round $t$. Let
    $\alpha(\gap|t) = q(\gap,t,y) / q(\gap,t) $
be the conditional probability that $R_0$ outputs a wrong answer given that it stops at round $t$. Note that by~\eqref{eq:lm-LB-infty} for each round $t$ it holds that
    $q(\gap,t)\to q(0,t)$ and $\alpha(\gap|t)\to \alpha(0|t)$
as $\gap\to 0$. Therefore for each round $t_0\in\N$ we have:
\begin{align*}
\rho    =    \textstyle \sum_{t\in\N}\; \alpha(\gap|t)\; q(\eps,t)
        \geq \textstyle \sum_{t\leq t_0 }\; \alpha(\gap|t) \; q(\eps,t)
        \to_{\gap\to\infty} \textstyle \sum_{t\leq t_0 }\; \alpha(0|t)\; q(0,t).
\end{align*}
Note that $\alpha(0|t) = \tfrac12$ by symmetry. It follows that
    $\sum_{t\leq t_0 }\; q(0,t) \leq 2\rho $
for each $t_0\in \N$. Therefore the probability that $R_0$ stops in finite time is
    $\sum_{t=1}^\infty\; q(0,t) \leq 2\rho $.
\end{proof}

\begin{proof}[Proof of Theorem~\ref{thm:LB}]
Suppose algorithm $\A$ is applied to an $\gap$-simple instance of the \problem. To simplify the notation, assume that crowd $1$ is the crowd with gap $\gap$ (and all other crowds have gap $0$).

Let $R_{(i)}$ be the instance of $R_0$ that corresponds to a given crowd $i$. Denote the composite stopping rule by $R$. Let $\sigma_R$ be the stopping time of $R$: the round in which $R$ stops.

For the following two definitions, let us consider an execution of algorithm $\A$ that runs forever (i.e., it keeps running even after $R$ decides to stop). First, let $\tau_i$ be the ``local'' stopping time of $R_{(i)}$: the number of samples from crowd $i$ that $R_{(i)}$ inputs before it decides to stop. Second, let $\sigma_i$ be the ``global'' stopping time of $R_{(i)}$: the round when $R_{(i)}$ decides to stop.
Note that $\sigma_R = \min_i\, \sigma_i$.

Let us use Lemma~\ref{lm:LB-infty} to show that $R$ stops essentially when $R_{(1)}$ tells it to stop. Namely:
\begin{align}\label{eq:pf-LB-sigma}
    \E[\sigma_1]\; (1-2k\rho) \leq \E[\sigma_R].
\end{align}
To prove~\eqref{eq:pf-LB-sigma}, consider the event
    $E \triangleq \{\min_{i>1} \tau_i =\infty\}$,
and let $1_E$ be the indicator variable of this event. Note that
    $\sigma_R \geq \sigma_1\, 1_E$
and that random variables $\sigma_1$ and $1_E$ are independent. It follows that
    $ \E[\sigma_R] \geq \Pr[E]\,\E[\sigma_1]$.
Finally, Lemma~\ref{lm:LB-infty} implies that $\Pr[E] \geq 1-2k\rho$. Claim proved.

Let $i_t$ be the option chosen by $\A$ in round $t$. Then by Wald's identity we have
\begin{align*}
\E[\tau_1]
    &= \E\left[ \sum_{t=1}^{\sigma_1} 1_{\{i_t=1\}} \right]
    = \E[1_{\{i_t=1\}}] \; \E[ \sigma_1]
    = \mu_1\; \E[\sigma_1] \\
\E[\cost(\A|R_0)]
    &= \E\left[\sum_{t=1}^{\sigma_R} c_{i_t} \right]
    = \E[c_{i_t}]\, \E[\sigma_R]
    = (\textstyle \sum_i c_i\,\mu_i)\; \E[\sigma_R] .
\end{align*}
Therefore, plugging in~\eqref{eq:pf-LB-sigma}, we obtain
\begin{align*}
\frac{\E[\cost(\A|R_0)]}{c_1\,\E[\tau_1]}
    \geq \frac{\textstyle \sum_i c_i\,\mu_i}{c_1\,\mu_1}\; (1-2k\rho).
\end{align*}
It remains to observe that $c_1\,\E[\tau_1]$ is precisely the expected total cost of the deterministic benchmark.
\end{proof}

\xhdr{Competitive ratio of $\AlgUCB$.}
Consider the case of two options and uniform  costs. Then (assuming $R_0$ is gap-decreasing) the approximate best crowd $i^*$ in Theorem~\ref{thm:UCB} is the best crowd. The competitive ratio of $\AlgUCB$ is, in the notation of Theorem~\ref{thm:UCB}, at most
    $1 + \frac{O(\Lambda \log T)}{\cost(i^*|R_0)}$.
This factor is close to $1$ when $R_0$ is tuned so as to decrease the error rate at the expense of increasing the expected running time.

\section{The randomized benchmark}
\label{sec:randomized-benchmark}

In this section we further discuss the randomized benchmark for crowd-selection algorithms. Informally, it is the best \emph{randomized} time-invariant policy given the latent information (response distributions of the crowds). Formally this benchmark is defined as $\min \cost(\mu|R_0)$, where the minimum is over all distributions $\mu$ over crowds, and $R_0$ is a fixed single-parameter stopping rule. Recall that in the definition of $\cost(\mu|R_0)$, the total crowd is treated as a single data source to which $R_0$ is applied.

 The total crowd under a given $\mu$ behaves as a single crowd whose response distribution $\D_\mu$ is given by
    $\D_\mu(x) = \E_{i\sim \mu}[\D_i(x)]$
for all options $x$. The gap of $\D_\mu$ will henceforth be called the \emph{induced gap} of $\mu$, and denoted $f(\mu) = \gap(\D_\mu)$. If the costs are uniform then $\cost(\mu|R_0)$ is simply the expected stopping time of $R_0$ on $\D_\mu$, which we denote $\tau(\D_\mu)$. Informally, $\tau(\D_\mu)$ is driven by the induced gap of $\mu$.

\OMIT{
\footnote{To make this a formal statement, we need to assume that the expected stopping time of $R_0$ for a gap-$\gap$ crowd is $\gap^{-2}$ up to a constant factor. (The factor may depend on $R_0$ but not on the response distribution of the crowd). Then
    $\cost(\mu|R_0) = \Theta(\gap^{-2})$,
where $\gap = \gap(\D_\mu)$.}
}

We show that the induced gap can be much larger than the gap of any crowd.

\begin{lemma}\label{lm:induced-gap}
Let $\mu$ be the uniform distribution over crowds. For any $\gap>0$ there exists a problem instance such that the gap of each crowd is $\gap$, and the induced gap of $\mu$ is at least $\tfrac{1}{10}$.
\end{lemma}
\begin{proof}
The problem instance is quite simple: there are two crowds and three options, and the response distributions are
    $(\tfrac25+\eps, \tfrac25,\tfrac15-\eps )$
and
    $(\tfrac25+\eps, \tfrac15-\eps, \tfrac25)$.
Then
    $\D_\mu = (\tfrac25 + \eps, \tfrac{3}{10}-\tfrac{\eps}{2}, \tfrac{3}{10}-\tfrac{\eps}{2})$.
\end{proof}

We conclude that the randomized benchmark does not reduce to the deterministic benchmark: in fact, it can be much stronger. Formally, this follows from Lemma~\ref{lm:induced-gap} under a very mild assumption on  $R_0$: that for any response distribution $\D$ with gap $\tfrac{1}{10}$ or more, and any response distribution $\D'$ whose gap is sufficiently small, it holds that $\tau(\D)\gg \tau(\D')$. The implication for the design of crowd-selection algorithms is that algorithms that zoom in on the best crowd may be drastically suboptimal. Instead, for some problem instances the right goal is to optimize over distributions over crowds.

However, the randomized benchmark coincides with the deterministic benchmark for some important special cases. First, the two benchmarks coincide if the costs are uniform and all crowds agree on the top \emph{two} options (and $R_0$ is gap-decreasing). Second, the two benchmarks may coincide if there are only two options ($|\options|=2$), see Lemma~\ref{lm:benchmarks-2options} below. To prove this lemma for non-uniform  costs, one needs to explicitly consider $\cost(\mu|R_0)$ rather than just argue about the induced gaps. Our proof assumes that the expected stopping time of $R_0$ is a concave function of the gap; it is not clear whether this assumption is necessary.

\begin{lemma}\label{lm:benchmarks-2options}
Consider the \problem with two options ($|\options|=2$). Consider a symmetric single-crowd stopping rule $R_0$. Assume that the expected stopping time of $R_0$ on response distribution $\D$ is a concave function of $\gap(\D)$.  Then the randomized benchmark coincides with the deterministic benchmark. That is,
    $\cost(\mu|R_0) \geq \min_i \cost(i|R_0)$
for any distribution $\mu$ over crowds.
\end{lemma}

\begin{proof}
Let $\mu$ be an arbitrary distribution over crowds. Recall that $f(\mu)$ denotes the induced gap of $\mu$. Note that
    $f(\mu) = \mu\cdot \vec{\gap}$.
To see this, let
    $\options = \{x,y\}$,
where $x$ is the correct answer, and write
\begin{align*}
\gap(\D_\mu)
    &= \D_\mu(x) - \D_\mu(y)
    = \mu\cdot \vec{D}(x) - \mu\cdot \vec{D}(y)
    = \mu \cdot \left( \vec{D}(x) - \vec{D}(y) \right)
    = \mu\cdot \vec{\gap}.
\end{align*}

Let $\A$ be the non-adaptive crowd-selection algorithm that corresponds to $\mu$. For each round $t$, let $i_t$ be the crowd chosen by $\A$ in this round, i.e. an independent sample from $\mu$. Let $N$ be the realized stopping time of $\A$. Let $\tau(\gap)$ be the expected stopping time of $R_0$ on response distribution with gap $\gap$. Note that
    $\E[N] = \tau( f(\mu) )$.
Therefore:
\begin{align*}
\cost(\mu|R_0)
    &= \textstyle \E\left[ \sum_{t=1}^N c_{i_t} \right]
    = \E[c_{i_t}] \; \E[N] & \text{by Wald's identity}\\
    &= (\vec{c}\cdot \mu)\;\;  \tau(\vec{\gap}\cdot\mu)
    \geq (\vec{c}\cdot \mu)\, \textstyle \sum_i\; \mu_i\, \tau(\gap_i)
        & \text{by concavity of $\tau(\cdot)$} \\
    &\geq  \min_i\; c_i\, \tau(\gap_i)
        & \text{by Claim~\ref{cl:standard-inequality}}\\
    &= \min_i \cost(i | R_0 ).
\end{align*}
We have used a general fact that
 $(\vec{x}\cdot \vec{\alpha})(\vec{x}\cdot \vec{\beta}) \geq \min_i\alpha_i \beta_i$
for any vectors $\vec{\alpha},\vec{\beta}\in \R^k_+$ and any $k$-dimensional distribution $\vec{x}$. A self-contained proof of this fact can be found in the Appendix (Claim~\ref{cl:standard-inequality}).
\end{proof}

\section{Crowd selection against the randomized benchmark}
\label{sec:randomized-benchmark-algs}

\newcommand{\mM}{\mathcal{M}}

We design a crowd-selection algorithm with guarantees against the randomized benchmark. We focus on uniform  costs, and (a version of) the single-crowd stopping rule from Section~\ref{sec:single-crowd}.

Our single-crowd stopping rule $R_0$ is as follows. Let $\empir{\gap}{*}{t}$ be the empirical gap of the total crowd. Then  $R_0$ stops upon reaching round $t$ if and only if
\begin{align}\label{eq:stopping-rule-total}
 \empir{\gap}{*}{t} > \errorC/\sqrt{t} \quad \text{or} \quad  t=T.
\end{align}
Here $\errorC$ is the ``quality parameter'' and $T$ is a given time horizon.

Throughout this section, let $\mM$ be the set of all distributions over crowds, and let $f^* = \max_{\mu\in \mM} f(\mu)$ be the maximal induced gap. The benchmark cost is then at least $\Omega((f^*)^{-2})$.

We design an algorithm $\A$ such that $\cost(\A|R_0)$ is upper-bounded by (essentially) a function of $f^*$, namely
    $O\left( (f^*)^{-(k+2)} \right)$.
We interpret this guarantee as follows: we match the benchmark cost for a distribution over crowds whose induced gap is
    $(f^*)^{2/(k+2)}$.
By Lemma~\ref{lm:induced-gap},  the gap of the best crowd may be much smaller, so this is can be a significant improvement over the deterministic benchmark.

\begin{theorem}\label{thm:AlgUnif}
Consider the \problem with uniform costs. Let $R_0$ be the single-crowd stopping rule given by~\refeq{eq:stopping-rule-total}.
There exists a crowd-selection algorithm $\A$ such that
$$ \cost(\A|R_0) \leq O\left( (f^*)^{-(k+2)}\; \sqrt{\log T} \right).
$$
\end{theorem}

The proof of Theorem~\ref{thm:AlgUnif} relies on some properties of the induced gap: concavity and Lipschitz-continuity. Concavity is needed for the reduction lemma (Lemma~\ref{lm:inducedMAB-reduction}), and Lipschitz-continuity is used to solve the MAB problem that we reduce to.

\begin{claim}\label{lm:inducedGap-props}
Consider the induced gap $f(\mu)$ as a function on $\mM\subset \Re^k_+$. First, $f(\mu)$ is a concave function. Second,
    $ |f(\mu) - f(\mu')| \leq n\,\|\mu-\mu'\|_1$
for any two distributions $\mu_1,\mu_2\in \mM$.
\end{claim}

\begin{proof}
Let $\mu$ be a distribution over crowds. Then
\begin{align}
f(\mu)
    = \D_\mu(x^*) - \max_{x\in \options \setminus \{x^*\}} \D_\mu(x)
    = \min_{x\in \options \setminus \{x^*\}}
    \mu\cdot \left( \vec{D}(x^*) - \vec{D}(x)\right).
        \label{eq:min-of-linear-fns}
\end{align}
Thus, $f(\mu)$ is concave as a minimum of concave functions. The second claim follows because
$$
(\mu-\mu') \cdot \left( \vec{D}(x^*) - \vec{D}(x)\right)
    \leq n\,\|\mu-\mu\|_1 \quad
    \text{for each option $x$}. \qedhere
$$
\end{proof}

\subsection{Proof of Theorem~\ref{thm:AlgUnif}}

\xhdr{Virtual rewards.}
Consider the MAB problem with virtual rewards, where arms correspond to distributions $\mu$ over crowds, and the virtual reward is equal to the induced gap $f(\mu)$; call it the \emph{induced MAB problem}. The standard definition of regret is with respect to the best fixed arm, i.e. with respect to $f^*$. We interpret an algorithm $\A$ for the induced MAB problem as a crowd-selection algorithm: in each round $t$, the crowd is sampled independently at random from  the distribution $\mu_t\in \mM$ chosen by $\A$.

\begin{lemma}\label{lm:inducedMAB-reduction}
Consider the \problem with uniform costs. Let $R_0$ be the single-crowd stopping rule given by~\refeq{eq:stopping-rule-total}. Let $\A$ be an MAB algorithm for the induced MAB instance. Suppose $\A$ has regret
    $O(t^{1-\gamma} \log T)$
with probability at least $1-\tfrac{1}{T}$, where
    $\gamma\in (0,\tfrac12]$.
Then
$$ \cost(\A|R_0) \leq O\left( (f^*)^{-1/\gamma}\; \sqrt{\log T} \right).
$$
\end{lemma}

\newcommand{\empirD}{\empirQty{\D}}
\newcommand{\empirG}{\empirQty{\gap}}

\begin{proof}
Let $\mu_t\in \mM$ be the distribution chosen by $\A$ is round $t$. Then the total crowd returns each option $x$ with probability $\mu_t\cdot\vec{D}(x)$, and this event is conditionally independent of the previous rounds given $\mu_t$.

Fix round $t$. Let $N_t(x)$ be the number times option $x$ is returned up to time $t$ by the total crowd, and let
    $ \empirD_t(x) = \tfrac{1}{t}\, N_t(x)$
be the corresponding empirical frequency. Note that
$$ \E\left[ \empirD_t(x) \right] = \bar{\mu}_t\cdot \vec{D}(x),
    \quad\text{where } \bar{\mu}_t \triangleq \frac{1}{t}\sum_{s=0}^t \mu_s.
$$

The time-averaged distribution over crowds $\bar{\mu}_t$ is a crucial object that we will focus on from here onwards. By Azuma-Hoeffding inequality, for each $C>0$ and each option $x\in\options$ we have
\begin{align}\label{eq:totalCrowd-Azuma}
\Pr\left[ \left|\empirD_t(x) -
 \bar{\mu}_t\cdot \vec{D}(x) \right|
    < \frac{C}{\sqrt{t}} \right] > 1-e^{-\Omega(C^2)}.
\end{align}
Let
    $\empirG_t = \gap(\empirD_t)$
be the empirical gap of the total crowd. Taking the Union Bound in \eqref{eq:totalCrowd-Azuma} over all options $x\in \options$, we conclude that $\empirG_t$ is close to the induced gap of $\bar{\mu}_t$:
$$ \Pr\left[ \left| \empirG_t - f(\bar{\mu}_t) \right|
    < \frac{C}{\sqrt{t}} \right] > 1-n\,e^{-\Omega(C^2)},\quad
    \text{for each $C>0$.}
$$
In particular, $R_0$ stops at round $t$ with probability at least $1-\frac{1}{T}$ as long as
\begin{align}\label{eq:inducedMAB-stopping}
    f(\bar{\mu}_t) > t^{-1/2}\;( \errorC+ O(\sqrt{\log T})).
\end{align}

By concavity of $f$, we have
    $ f(\bar{\mu}_t) \geq \bar{f}_t $,
where
    $\bar{f}_t \triangleq \frac{1}{t}\sum_{s=0}^t f(\mu_s)$
is the time-averaged virtual reward. Now, $ t\bar{f}_t$ is simply the total virtual reward by time $t$, which is close to $f^*$ with high probability. Specifically, the regret of $\A$ by time $t$ is
    $ R(t) = t(f^* - \bar{f}_t) $,
and we are given a high-probability upper bound on $R(t)$.

Putting this all together,
    $ f(\bar{\mu}_t) \geq \bar{f}_t \geq f^* - R(t)/t $.
An easy computation shows that
    $f(\bar{\mu}_t)$
becomes sufficiently large to trigger the stopping condition~\refeq{eq:inducedMAB-stopping} for
    $t = O\left( (f^*)^{-1/\gamma}\; \sqrt{\log T} \right)$.
\end{proof}

\xhdr{Solving the induced MAB problem.}
We derive a (possibly inefficient) algorithm for the induced MAB instance. We treat $\mM$ as a subset of $\Re^k$, endowed with a metric
    $d(\mu,\mu') = n\,\|\mu-\mu'\|_1$.
By Lemma~\ref{lm:inducedGap-props}, the induced gap $f(\mu)$ is Lipschitz-continuous with respect to this metric. Thus, in the induced MAB problem arms form a metric space $(\mM,d)$ such that the (expected) rewards are Lipschitz-continuous for this metric space. MAB problems with this property are called \emph{Lipschitz MAB}~\cite{LipschitzMAB-stoc08}.

We need an algorithm for Lipschitz MAB that works with virtual rewards.
We use the following simple algorithm from \cite{Bobby-nips04,LipschitzMAB-stoc08}. We treat $\mM$ as a subset of $\Re^k$, and apply this algorithm to $\Re^k$. The algorithm runs in phases $j=1,2,3,\ldots$ of duration $2^j$. Each phase $j$ is as follows. For some fixed parameter $\delta_j>0$, discretize $\Re^k$ uniformly with granularity $\delta_j$. Let $S_j$ be the resulting set of arms. Run bandit algorithm $\UCB$~\cite{bandits-ucb1} on the arms in $S_j$. (For each arm in $S_j\setminus \mM$, assume that the reward is always $0$.) This completes the specification of the algorithm.

Crucially, we can implement $\UCB$ (and therefore the entire uniform algorithm) with virtual rewards, by using $\empirG_t$ as an estimate for $f(\mu)$. Call the resulting crowd-selection algorithm \AlgUnif.

Optimizing the $\delta_j$ using a simple argument from \cite{Bobby-nips04},  we obtain regret
    $O(t^{1-1/(k+2)}\, \log T)$
with probability at least $(1-\tfrac{1}{T})$. Therefore, by Lemma~\ref{lm:inducedMAB-reduction} $\cost(\AlgUnif|R_0)$ suffices to prove Theorem~\ref{thm:AlgUnif}.

We can also use a more sophisticated \emph{zooming algorithm} from~\cite{LipschitzMAB-stoc08}, which obtains the same in the worst case, but achieves better regret for ``nice'' problem instances. This algorithm also can be implemented for virtual rewards (in a similar way). However, it is not clear how to translate the improved regret bound for the zooming algorithm into a better cost bound for the \problem.

\OMIT{
\begin{lemma}
Let $R_0$ be the single-crowd stopping rule given by~\refeq{eq:stopping-rule-total}. Let $\D$ be a response distribution with gap $\gap>0$. Let $\tau(\D)$ be the expected stopping time of $R_0$ on $\D$. Then
$\tfrac{1}{c}\, \errorC\, \gap^{-2}
    \leq \tau(\D)
    \leq c\, \errorC\, \gap^{-2}$
as long as
$\errorC > c
    \sqrt{\log(T+\tfrac{n}{\gap} \log \tfrac{n}{\gap})}$,
for a sufficiently large absolute constant $c$.
\end{lemma}
}

\section{The bi-criteria objective}
\label{sec-bicriteria}

In this section we state our results with respect to the bi-criteria objective, for both deterministic and randomized benchmarks. Recall that our bi-criteria objective focuses on the worst-case error rates.

We only consider the case of uniform costs. Let $k\geq 2$ be the number of crowds.

\xhdr{Worst-case error rates.}
Let $R_0$ be a single-crowd stopping rule. Let $\error(R_0)$ be the worst-case error rate of $R_0$, taken over all single-crowd instances (i.e., all values of the gap).

Let $R$ be the composite stopping rule based on $R_0$. Let $(\A,R_0)$ denote the \Algorithm in which a crowd-selection algorithm $\A$ is used  together with the stopping rule $R$. Let $\error(\A|R_0)$ be the worst-case error rate of $(\A,R_0)$, over all problem instances. Then
\begin{align}\label{eq:A-error}
 \error(\A|R_0) \leq (k+1)\, \error(R_0).
\end{align}

Note that the worst-case error rate of  benchmark is simply $\error(R_0)$. (It is achieved on a problem instance in which all crowds have gap which maximizes the error rate of $R_0$.) Thus, using the same $R_0$ roughly equalizes the worst-case error rate between $\A$ and the benchmarks.

\xhdr{Absolute benchmarks.}
We consider benchmarks in which both the best crowd (resp., the best distribution over crowds) and the stopping rule are chosen by the benchmark. Thus, the benchmark cost is not relative to any particular single-crowd stopping rule. We call such benchmarks \emph{absolute}.

Let $T(\rho)$ be the smallest time horizon $T$ for which the single-crowd stopping rule in \eqref{eq:stopping-rule-total-body} achieves $\error(R_0)\leq \rho$. Fix error rate $\rho>0$ and time horizon $T\geq T(\rho)$. We focus on symmetric, gap-decreasing single-crowd stopping rules $R_0$ such that $\error(R_0)\leq \rho$
and $R_0$ must stop after $T$ rounds; let $\mathcal{R}(\rho,T)$ be the family of all such stopping rules.

Fix a problem instance. Let $i^*$ be the crowd with the largest bias, and let $\mu^*$ be the distribution over crowds with the largest induced bias. The \emph{absolute deterministic benchmark} (with error rate $\rho$ and time horizon $T\geq T(\rho)$) is defined as
$$ \bench(i^*,\rho,T)
    =  \min_{R_0\in \mathcal{R}(\rho,T)} \cost(i^*|R_0).
$$
Likewise, the \emph{absolute randomized benchmark} is defined as
$$ \bench(\mu^*,\rho,T)
    =  \min_{R_0\in \mathcal{R}(\rho,T)} \cost(\mu^*|R_0).
$$

\begin{theorem}[bi-criteria results]
\label{thm:UCB-bi}
Consider the \problem with $k$ crowds and uniform costs. Fix error rate $\rho>0$ and time horizon $T\geq T(\rho)$. Then:

\begin{itemize}
\item[(a)] {\em Deterministic benchmark.}
There exists a \Algorithm $(\A,R_0)$ such that
\begin{align*}
\cost(\A|R_0)  &\leq \bench(i^*,\rho,T) + O(\Lambda \log T),
    \text{ where }
    \Lambda = \textstyle \sum_{i\neq i^*}
        \left(\bias_{i^*}-\bias_i \right)^{-2} ,\\
\error(\A|R_0) &\leq (k+1)\,\rho.
\end{align*}

\item[(b)] {\em Randomized benchmark.}
There exists a \Algorithm $(\A,R_0)$ such that
\begin{align*}
\cost(\A|R_0)
    &\leq O(\log T \log \tfrac{1}{\rho})
     \;(\bench(\mu^*,\rho,T))^{1+k/2} \\
   \error(\A|R_0) &\leq (k+1)\,\rho.
\end{align*}
\end{itemize}
\end{theorem}

\begin{myproof}[Sketch]
For part (a), we use the version of $\AlgUCB$ as in Theorem~\ref{thm:UCB}, with the single-crowd stopping rule $R_0$ from the absolute deterministic benchmark. The upper bound on $\cost(\A|R_0)$ follows from Theorem~\ref{thm:UCB}. The upper bound on $\error(\A|R_0)$ follows from \eqref{eq:A-error}.

For part (b), we use the algorithm from Theorem~\ref{thm:AlgUnif}, together with the stopping rule given by
\eqref{eq:stopping-rule-total}. The stopping rule has time horizon $T$; the quality parameter $\errorC$ is tuned so that the worst-case error rate matches that in the absolute randomized benchmark. The upper bound on $\cost(\A|R_0)$ follows from Theorem~\ref{thm:AlgUnif}, and the upper and lower bounds in Section~\ref{sec:single-crowd}. The upper bound on $\error(\A|R_0)$ follows from \eqref{eq:A-error}.
\end{myproof}

\xhdr{A lower bound on the error rate.}
Fix a single-crowd stopping rule $R_0$ with $\rho = \error(R_0)$, and a crowd-selection algorithm $\A$. To complement \eqref{eq:A-error}, we conjecture that
    $\error(\A|R_0) \geq \rho$.
We prove a slightly weaker result: essentially, if the composite stopping rule does not use the total crowd, then
    $\error(\A|R_0) \geq \rho\,(1-2k\rho)$.

We will need a mild assumption on $\A$: essentially, that it never commits to stop using any given crowd. Formally, $\A$ is called \emph{non-committing} if for every problem instance, each time $t$, and every crowd $i$, it will choose crowd $i$ at some time after $t$ with probability one. (Here we consider a run of $\A$ that continues indefinitely, without being stopped by the stopping rule.)

\begin{lemma}\label{lm:error-rate-LB}
Let $R_0$ be a symmetric single-crowd stopping rule with worst-case error rate $\rho$. Let $\A$ be a non-committing crowd-selection algorithm, and let $R$ be the composite stopping rule based on $R_0$ which does not use the total crowd. If $\A$ is used in conjunction with $R$, the worst-case error rate is at least $\rho\,(1-2k\rho)$, where $k$ is the number of crowds.
\end{lemma}

\begin{myproof}
Suppose $R_0$ attains the worst-case error rate for a crowd with gap $\gap$. Consider the problem instance in which one crowd (say, crowd $1$) has gap $\gap$ and all other crowds have gap $0$. Let $R_{(i)}$ be the instance of $R_0$ that takes inputs from crowd $i$, for each $i$. Let $E$ be the event that each $R_{(i)}$, $i>1$ does not ever stop. Let $E'$ be the event that $R_{(1)}$ stops and makes a mistake. These two events are independent, so the error rate of $R$ is at least $\Pr[E]\, \Pr[E']$. By the choice of the problem instance, $\Pr[E']=\rho$. And by Lemma~\ref{lm:LB-infty}, $\Pr[E] \geq 1-2k\rho$. It follows that the error rate of $R$ is at least $\rho\, (1-2k\rho)$.
\end{myproof}

\section{Experimental results: single crowd}
\label{sec:expts-single-crowd}

We conduct two experiments. First, we analyze real-life workloads to find which gaps are typical for response distributions that arise in practice. Second, to study the performance of the single-crowd stopping rule suggested in Section~\ref{sec:single-crowd}, using a large-scale simulation with a realistic distribution of gaps. We are mainly interested in the tradeoff between the error rate and the expected stopping time. We find that this tradeoff is acceptable in practice.

\xhdr{Typical gaps in real-life workloads.}
We analyze several batches of microtasks  extracted from a commercial crowdsourcing platform (approx. 3000 microtasks total). Each batch consists of microtasks of the same type, with the same instructions for the workers.  Most microtasks are related to relevance assessments for a web search engine. Each microtask was given to at least 50 judges coming from the same ``crowd".

In every batch, the empirical gaps of the microtasks are very close to being \emph{uniformly distributed} over the range. A practical take-away is that assuming a Bayesian prior on the gap would not be very helpful, which justifies and motivates our modeling choice not to assume Bayesian priors. In Figure~\ref{fig:CDF-gap}, we provide CDF plots for two of the batches; the plots for the other batches are similar.

\begin{figure}[h]
        \centering
        \begin{subfigure}[b]{0.5\textwidth}
                \centering
\begin{tikzpicture}
\begin{axis}[
    width=6cm,
]
\addplot [color=blue, mark=.] table [x=index, y=gap] {chart_gap_distribution.data};
\addplot [color=black] coordinates { (0,0.1832) (128,1.0664) }
[yshift=-25pt] node[pos=0.5] {$R^2 = 0.9215$};
\end{axis}
\end{tikzpicture}
                \caption{Batch 1: 128 microtasks, 2 options each}
                \label{fig:CDF-gap-a}
        \end{subfigure}%
        \begin{subfigure}[b]{0.5\textwidth}
                \centering
\begin{tikzpicture}
\begin{axis}[
    width=6cm,
]
\addplot [color=blue, mark=.] table [x=index, y=gap] {chart_aggregate_gap_distribution.data};
\addplot [color=black] coordinates { (0,0.1614) (604,1.1278) }
[yshift=-25pt] node[pos=0.5] {$R^2 = 0.9433$};
\end{axis}
\end{tikzpicture}
                \caption{Batch 2: 604 microtasks, variable \#options}
                \label{fig:CDF-gap-b}
        \end{subfigure}
        \caption{CDF for the empirical gap in real-life workloads.}\label{fig:CDF-gap}
\end{figure}
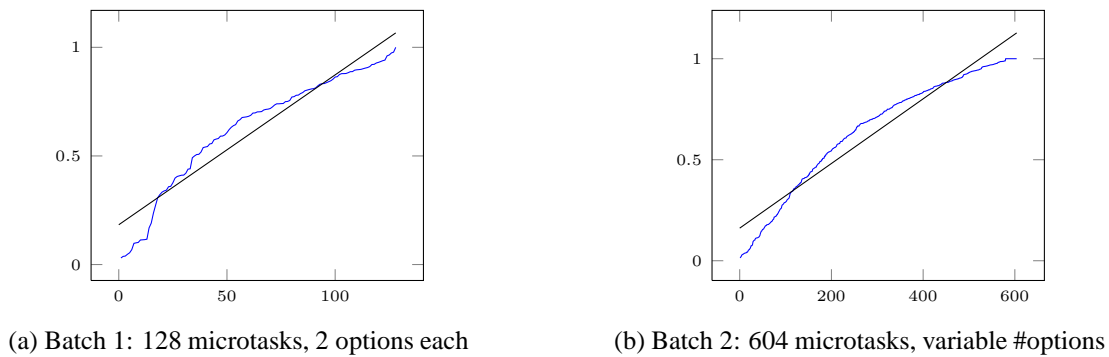

\xhdr{Our single-crowd stopping rule on simulated workloads.}
We study the performance of the single-crowd stopping rule suggested in Section~\ref{sec:single-crowd}. Our simulated workload consists of 10,000 microtasks with two options each. For each microtask, 
the gap is is chosen independently and uniformly at random in the range $[0.05, 1]$. This distribution of gaps is realistic according to the previous experiment. (Since there are only two options the gap fully describes the response distribution.)

We vary the parameter $\errorC$ and for each $\errorC$ we measure the average total cost (i.e., the stopping time averaged over all microtasks) and the error rate. The results are reported in Figure~\ref{fig:synthetic-singleCrowd}. In particular, for this workload, an error rate of $< 5\%$ can be obtained with an average of $<8$ workers per microtask.

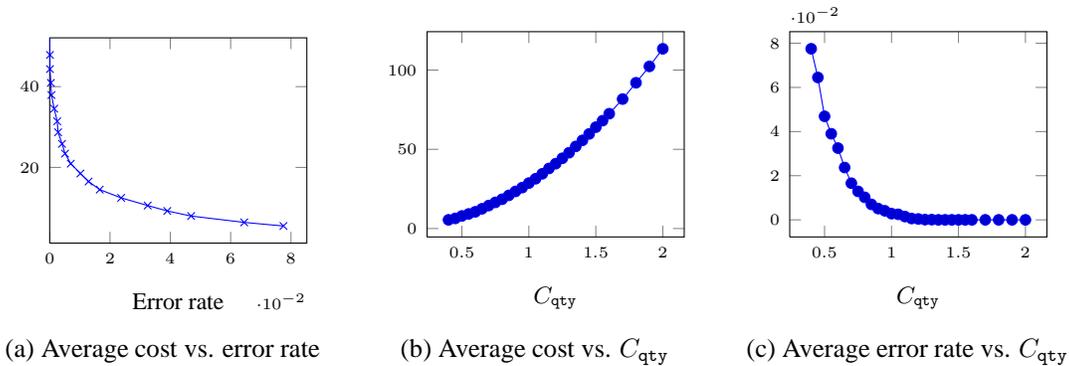
\begin{figure}[h]
        \centering
        \begin{subfigure}[b]{0.3\textwidth}
                \centering
                \begin{tikzpicture}
\begin{axis}[
    width=5cm,
        xlabel=Error rate,
        legend pos=north east,
        xmin=0.00001,
]
\addplot [color=blue, mark=x] table [x=adap_errors, y=adap_cost] {chart_uniform_bias_cost_v_error.data};
\end{axis}
\end{tikzpicture}
                \caption{Average cost vs. error rate}
        \end{subfigure}%
        \begin{subfigure}[b]{0.3\textwidth}
                \centering
               \begin{tikzpicture}
\begin{axis}[
        width=5cm,
        xlabel=$\errorC$,
        legend pos=north east
]
\addplot table [x=Threshold, y=adap_cost] {chart_uniform_bias_cost_v_error.data};
\end{axis}
\end{tikzpicture}
                \caption{Average cost vs. $\errorC$}
        \end{subfigure}%
        \begin{subfigure}[b]{0.3\textwidth}
                \centering
             \begin{tikzpicture}
\begin{axis}[
        width=5cm,
        xlabel=$\errorC$,
        legend pos=north east
]
\addplot table [x=Threshold, y=adap_errors] {chart_uniform_bias_cost_v_error.data};
\end{axis}
\end{tikzpicture}
                \caption{Average error rate vs. $\errorC$}
        \end{subfigure}
        \caption{Our single-crowd stopping rule on the synthetic workload.}\label{fig:synthetic-singleCrowd}
\end{figure}

Our stopping rule adapts to the gap of the microtask: it uses only a few workers for easy microtasks (ones with a large gap), and more workers for harder microtasks (those with a small gap). In particular, we find that our stopping rule requires significantly smaller number of workers than a non-adaptive stopping rule: one that always uses the same number of workers while ensuring a desired error rate.

\section{Experimental results: crowd-selection algorithms}
\label{sec:expts-multi-crowds}

We study the experimental performance of the various crowd-selection algorithms discussed in Section~\ref{sec:multi-crowd}. Specifically, we consider algorithms $\AlgUCB$ and $\AlgThompson$, and compare them to our straw-man solutions: $\AlgEER$ and $\AlgRR$.%
\footnote{In the plots, we use shorter names for the algorithms: respectively, $\AlgUCBshort$, $\AlgThompsonshort$, $\AlgEERshort$, and $\AlgRRshort$.}
Our goal is both to compare the different algorithms and to show that the associated costs are practical.  We find that \AlgEER consistently outperforms \AlgRR for very small error rates, \AlgUCB significantly outperforms both across all error rates, and \AlgThompson significantly outperforms all three.

We use all crowd-selection algorithms in conjunction with the composite stopping rule based on the single-crowd stopping rule proposed Section~\ref{sec:single-crowd}. Recall that the stopping rule has a ``quality parameter" $\errorC$ which implicitly controls the tradeoff between the error rate and the expected stopping time.

We use three simulated workloads. All three workloads consist of microtasks with two options, three crowds, and unit costs. In the first workload, which we call the \emph{easy workload}, the crowds have
gaps $(0.3,0,0)$. That is, one crowd has gap $0.3$ (so it returns the correct answer with probability $0.8$), and the remaining two crowds have gap $0$ (so they provide no useful information). This is a relatively easy workload for our crowd-selection algorithms because the best crowd has a much larger gap than the other crowds, which makes the best crowd easier to identify. In the second workload, called the \emph{medium workload}, crowds have gaps $(0.3,0.1,0.1)$, and in the third workload, called the \emph{hard workload}, the crowds have gaps $(0.3,0.2,0.2)$. The third workload is hard(er) for the crowd-selection algorithms in the sense that the best crowd is hard(er) to identify, because its gap is not much larger than the gap of the other crowds. The order that the crowds are presented to the algorithms is randomized for each instance, but is kept the same across the different algorithms.

The quality of an algorithm is measured by the tradeoff between its average total cost and its error rate. To study this tradeoff, we vary the quality parameter $\errorC$ to obtain (essentially) any desired error rate. We compare the different algorithms by reporting the average total cost of each algorithm (over 20,000 runs with the same quality parameter) for a range of error rates. Specifically, for each error rate we report the average cost of each algorithm normalized to the average cost of the naive algorithm  \AlgRR (for the same error rate). See
Figure~\ref{fig:multiCrowd-main} for the main plot: the average cost vs. error rate plots for all three workloads. Additional results, reported in Figure~\ref{fig:multiCrowd-details} (see  page~\pageref{fig:multiCrowd-details}) show the raw average total costs and error rates for the range of values of the quality parameter $\errorC$.

\begin{figure}[h]
        \centering
        \begin{subfigure}[b]{0.32\textwidth}
                \centering
                \begin{tikzpicture}
\begin{axis}[
        width=6cm,
        xlabel=Error rate,
        legend pos=south east,
            xmin=0.001,
        ymin=0,
        xticklabel style={
                /pgf/number format/.cd,
                fixed,
                fixed zerofill,
                precision=2,
        },
]
\addplot [color=red,mark=x] table [x=3-ph_errors, y=3-ph_norm_cost] {chart_bias0_cost_v_error.data};
\addlegendentry{\AlgEERshort}
\addplot [color=blue,mark=x] table [x=UCB_errors, y=UCB_norm_cost] {chart_bias0_cost_v_error.data};
\addlegendentry{\AlgUCBshort}
\addplot [color=green,mark=x] table [x=Thom_errors, y=Thom_norm_cost] {chart_bias0_cost_v_error.data};
\addlegendentry{\AlgThompsonshort}
\end{axis}
\end{tikzpicture}
                \caption{Easy: gaps $(.3,0,0)$.}
        \end{subfigure}%
        \begin{subfigure}[b]{0.32\textwidth}
                \centering
                \begin{tikzpicture}
\begin{axis}[
        width=6cm,
        xlabel=Error rate,
        legend pos=south east,
            xmin=0.001,
        ymin=0,
        xticklabel style={
                /pgf/number format/.cd,
                fixed,
                fixed zerofill,
                precision=2,
        },
]
\addplot [color=red,mark=x] table [x=3-ph_errors, y=3-ph_norm_cost] {chart_bias1_cost_v_error.data};
\addlegendentry{\AlgEERshort}
\addplot [color=blue,mark=x] table [x=UCB_errors, y=UCB_norm_cost] {chart_bias1_cost_v_error.data};
\addlegendentry{\AlgUCBshort}
\addplot [color=green,mark=x] table [x=Thom_errors, y=Thom_norm_cost] {chart_bias1_cost_v_error.data};
\addlegendentry{\AlgThompsonshort}
\end{axis}
\end{tikzpicture}
                \caption{Medium: gaps $(.3,.1,.1)$.}
        \end{subfigure}%
        \begin{subfigure}[b]{0.32\textwidth}
                \centering
                \begin{tikzpicture}
\begin{axis}[
        width=6cm,
        xlabel=Error rate,
        legend pos=south east,
            xmin=0.001,
        ymin=0,
        xticklabel style={
                /pgf/number format/.cd,
                fixed,
                fixed zerofill,
                precision=2,
        },
]
\addplot [color=red,mark=x] table [x=3-ph_errors, y=3-ph_norm_cost] {chart_bias2_cost_v_error.data};
\addlegendentry{\AlgEERshort}
\addplot [color=blue,mark=x] table [x=UCB_errors, y=UCB_norm_cost] {chart_bias2_cost_v_error.data};
\addlegendentry{\AlgUCBshort}
\addplot [color=green,mark=x] table [x=Thom_errors, y=Thom_norm_cost] {chart_bias2_cost_v_error.data};
\addlegendentry{\AlgThompsonshort}
\end{axis}
\end{tikzpicture}
                \caption{Hard: gaps $(.3,.2,.2)$.}
        \end{subfigure}
        \caption{Crowd-selection algorithms: error rate vs.
            average total cost (relative to $\AlgRR$).}
        \label{fig:multiCrowd-main}
\end{figure}
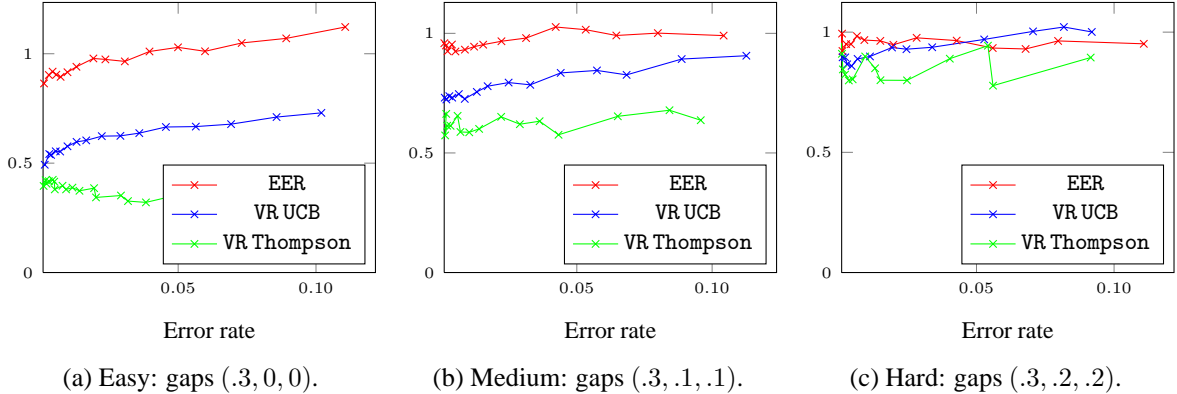

For \AlgUCB we tested different parameter values for the parameter $C$ which balances between exploration and exploitation. We obtained the best results for a range of workloads for $C=1$ and this is the value we use in all the experiments. For \AlgThompson we start with a uniform prior on each crowd.

\begin{figure}[p]
        \centering
        {\bf \Large Additional plots for crowd-selection algorithms} \vspace{5mm}

        \begin{subfigure}[b]{0.5\textwidth}
                \centering
                \begin{tikzpicture}
\begin{axis}[
        width=7cm,
        xlabel=$\errorC$,
        ylabel=Average cost,
        legend pos=north west
]
\addplot table [x=Threshold, y=RR_cost] {chart_bias0_full.data};
\addlegendentry{\AlgRRshort}
\addplot table [x=Threshold, y=3-ph_cost] {chart_bias0_full.data};
\addlegendentry{\AlgEERshort}
\addplot table [x=Threshold, y=UCB_cost] {chart_bias0_full.data};
\addlegendentry{\AlgUCBshort}
\addplot table [x=Threshold, y=Thom_cost] {chart_bias0_full.data};
\addlegendentry{\AlgThompsonshort}
\end{axis}
\end{tikzpicture}
        \end{subfigure}%
        \begin{subfigure}[b]{0.5\textwidth}
                \centering
                \begin{tikzpicture}
\begin{axis}[
        width=7cm,
        xlabel=$\errorC$,
        ylabel=Error rate,
        legend pos=north east
]
\addplot table [x=Threshold, y=RR_errors] {chart_bias0_full.data};
\addlegendentry{\AlgRRshort}
\addplot table [x=Threshold, y=3-ph_errors] {chart_bias0_full.data};
\addlegendentry{\AlgEERshort}
\addplot table [x=Threshold, y=UCB_errors] {chart_bias0_full.data};
\addlegendentry{\AlgUCBshort}
\addplot table [x=Threshold, y=Thom_errors] {chart_bias0_full.data};
\addlegendentry{\AlgThompsonshort}
\end{axis}
\end{tikzpicture}
        \end{subfigure}%

        The easy workload: gaps $(.3,0,0)$. \vspace{2mm}

        \centering
        \begin{subfigure}[b]{0.5\textwidth}
                \centering
                \begin{tikzpicture}
\begin{axis}[
        width=7cm,
        xlabel=$\errorC$,
        ylabel=Average cost,
        legend pos=north west
]
\addplot table [x=Threshold, y=RR_cost] {chart_bias1_full.data};
\addlegendentry{\AlgRRshort}
\addplot table [x=Threshold, y=3-ph_cost] {chart_bias1_full.data};
\addlegendentry{\AlgEERshort}
\addplot table [x=Threshold, y=UCB_cost] {chart_bias1_full.data};
\addlegendentry{\AlgUCBshort}
\addplot table [x=Threshold, y=Thom_cost] {chart_bias1_full.data};
\addlegendentry{\AlgThompsonshort}
\end{axis}
\end{tikzpicture}
        \end{subfigure}%
        \begin{subfigure}[b]{0.5\textwidth}
                \centering
                \begin{tikzpicture}
\begin{axis}[
        width=7cm,
        xlabel=$\errorC$,
        ylabel=Error rate,
        legend pos=north east
]
\addplot table [x=Threshold, y=RR_errors] {chart_bias1_full.data};
\addlegendentry{\AlgRRshort}
\addplot table [x=Threshold, y=3-ph_errors] {chart_bias1_full.data};
\addlegendentry{\AlgEERshort}
\addplot table [x=Threshold, y=UCB_errors] {chart_bias1_full.data};
\addlegendentry{\AlgUCBshort}
\addplot table [x=Threshold, y=Thom_errors] {chart_bias1_full.data};
\addlegendentry{\AlgThompsonshort}
\end{axis}
\end{tikzpicture}
        \end{subfigure}%

        The medium workload: gaps $(.3,.1,.1)$. \vspace{2mm}

        \centering
        \begin{subfigure}[b]{0.5\textwidth}
                \centering
                \begin{tikzpicture}
\begin{axis}[
        width=7cm,
        xlabel=$\errorC$,
        ylabel=Average cost,
        legend pos=north west
]
\addplot table [x=Threshold, y=RR_cost] {chart_bias2_full.data};
\addlegendentry{\AlgRRshort}
\addplot table [x=Threshold, y=3-ph_cost] {chart_bias2_full.data};
\addlegendentry{\AlgEERshort}
\addplot table [x=Threshold, y=UCB_cost] {chart_bias2_full.data};
\addlegendentry{\AlgUCBshort}
\addplot table [x=Threshold, y=Thom_cost] {chart_bias2_full.data};
\addlegendentry{\AlgThompsonshort}
\end{axis}
\end{tikzpicture}
        \end{subfigure}%
        \begin{subfigure}[b]{0.5\textwidth}
                \centering
                \begin{tikzpicture}
\begin{axis}[
        width=7cm,
        xlabel=$\errorC$,
        ylabel=Error rate,
        legend pos=north east
]
\addplot table [x=Threshold, y=RR_errors] {chart_bias2_full.data};
\addlegendentry{\AlgRRshort}
\addplot table [x=Threshold, y=3-ph_errors] {chart_bias2_full.data};
\addlegendentry{\AlgEERshort}
\addplot table [x=Threshold, y=UCB_errors] {chart_bias2_full.data};
\addlegendentry{\AlgUCBshort}
\addplot table [x=Threshold, y=Thom_errors] {chart_bias2_full.data};
\addlegendentry{\AlgThompsonshort}
\end{axis}
\end{tikzpicture}
        \end{subfigure}%

        The hard workload: gaps $(.3,.2,.2)$. \vspace{2mm}

        \caption{Crowd-selection algorithms:
            Average cost and error rate vs. $\errorC$.}

        \label{fig:multiCrowd-details}
\end{figure}

\xhdr{Results and discussion.}
For the easy workload the cost of \AlgUCB is about $60\%$ to $70\%$ of the cost of \AlgRR. \AlgThompson is significantly better, with a cost of about $40\%$ the cost of \AlgRR.
For the medium workload the cost of \AlgUCB is about $80\%$ to $90\%$ of the cost of \AlgRR. \AlgThompson is significantly better, with a cost of about $70\%$ the cost of \AlgRR.
For the hard workload the cost of \AlgUCB is about $90\%$ to $100\%$ of the cost of \AlgRR. \AlgThompson is better, with a cost of about $80\%$ to $90\%$ the cost of \AlgRR.
While our analysis predicts that \AlgEER should be (somewhat) better than \AlgRR, our experiments do not confirm this for every error rate.

As the gap of the other crowds approaches that of the best crowd, choosing the best crowd becomes less important, and so the advantage of the adaptive algorithms over \AlgRR diminishes. In the extreme case where all crowds have the same gap all the algorithms would perform the same with an error rate that depends on the stopping rule. We conclude that \AlgUCB provides an advantage, and \AlgThompson provides a significant advantage, over the naive scheme of \AlgRR.

\section{Related work}
\label{sec:related-work}

For general background on crowdsourcing and human computation, refer to \citet{Law11}. Most of the work on crowdsourcing is usually done using platforms like \emph{Amazon Mechanical Turk} or \emph{CrowdFlower}.
Results using those platforms have shown that majority voting is a good approach to achieve quality~\cite{Snow08}. Get Another Label~\cite{Sheng08} explores adaptive schemes for the single-crowd case under Baysian assumptions (while our focus is on multiple-crowds and regret under non-Bayesian uncertainty). A study on machine translation quality uses preference voting  for combining ranked judgments~\cite{Callison-Burch09}.
Vox Populi~\cite{Dekel09} suggests to prune low quality workers, however their approach is not adaptive and their analysis does not provide regret bounds (while our focus is on adaptively choosing which crowds to exploit and obtaining regret bounds against an optimal algorithm that knows the quality of each crowd).
Budget-Optimal Task Allocation~\cite{KOS11} focuses on a non-adaptive solution to the task allocation problem given a prior distribution on both tasks and judges (while we focus adaptive solutions and do not assume priors on judges or tasks).
From a methodology perspective, CrowdSynth  focuses on addressing consensus tasks by leveraging
supervised learning~\cite{Kamar12}.
Adding a crowdsourcing layer as part of a computation engine is a very recent line of research. An example is CrowdDB, a system for crowdsourcing which includes human computation for processing queries~\cite{Franklin11}. CrowdDB offers basic quality control features, but we expect adoption of more advanced techniques as those systems become more available within the community.

Multi-armed bandits (MAB) have a rich literature in Statistics, Operations Research, Computer Science and Economics. A proper discussion of this literature is beyond our scope; see \cite{CesaBL-book} for background. Most relevant to our setting is the work on prior-free MAB with stochastic rewards: \cite{Lai-Robbins-85,bandits-ucb1} and the follow-up work, and Thompson heuristic~\cite{Thompson-1933}. Recent work on Thompson heuristic includes \cite{Thompson-Bing-icml10,Thompson-Scott10,Thompson-nips11,Thompson-analysis-arxiv11}.

\asedit{
Our setting is superficially similar to \emph{budgeted MAB}, a version of MAB where the goal is to find the best arm after a fixed period of exploration (e.g., \cite{Tsitsiklis-bandits-04,Bubeck-alt09}). Likewise, there is some similarity with the work on \emph{budgeted active learning} (e.g. \cite{Lizotte-uai03,Madani-uai04,Kaplan-stoc05}), where an algorithm repeatedly chooses instances and receives correct labels for these instances, with a goal to eventually output the correct hypothesis. The difference is that in the \problem, an algorithm repeatedly chooses among \emph{crowds}, whereas in the end the goal is to pick the correct \emph{option}; moreover, the true ``reward" or ``label" for each chosen crowd is not revealed to the algorithm and is not even well-defined.}

Settings similar to stopping rules for a single crowd (but with somewhat different technical objectives) were considered in prior work, e.g. \cite{Bechhofer59,Ramey79,Bechhofer85,Dagum-sicomp00,Mnih-icml08}.

\asedit{
In a very recent concurrent and independent work, \cite{Jenn-aaai12,Jenn-icml13,Chen-icml13,Tran-aamas13} studied related, but technically incomparable settings. The first three papers consider adaptive task assignment with multiple tasks and a budget constraint on the total number or total cost of the workers. In \cite{Jenn-aaai12,Jenn-icml13} workers arrive over time, and the algorithm selects which tasks to assign. In \cite{Chen-icml13}, in each round the algorithm chooses a worker and a task, and Bayesian priors are available for the difficulty of each task and the skill level of each worker (whereas our setting is prior-independent). Finally, \citet{Tran-aamas13} studies a \emph{non-adaptive} task assignment problem where the algorithm needs to distribute a given budget across multiple tasks with known per-worker costs.
}

\section{Open questions}
\label{sec:questions}

\OMIT{
Research in human computation algorithms is a new and exciting area that has a lot of potential. In this paper, we presented the \problem and introduced a number of algorithms. Our approach was inspired by real world problems from a commercial search engine such as relevance assessment and training set construction.}

\xhdr{The \problem.} The main open questions concern crowd-selection algorithms for the randomized benchmark. First, we do not know how to handle non-uniform costs. Second, we conjecture that our algorithm for uniform costs can be significantly improved. Moreover, it is desirable to combine guarantees against the randomized benchmark with (better) guarantees against the deterministic benchmark.

Our results prompt several other open questions. First, while we obtain strong provable guarantees for $\AlgUCB$, it is desirable to extend these or similar guarantees to $\AlgThompson$, since this algorithm performs best in the experiments. Second, is it possible to significantly improve over the composite stopping rules? Third, is it advantageous to forego our "independent design" approach and design the crowd-selection algorithms jointly with the stopping rules?

\xhdr{Extended models.} It is tempting to extend our model in several directions listed below. First, while in our model the gap of each crowd does not change over time, it is natural to study settings with bounded or ``adversarial'' change; one could hope to take advantage of the tools developed for the corresponding versions of MAB. Second, as discussed in the introduction, an alternative model worth studying is to assign a monetary penalty to a mistake, and optimize the overall cost (i.e., cost of labor minus penalty). Third, one can combine the \problem with learning across multiple related microtasks.

\xhdr{Acknowledgements.}
We thank Ashwinkumar Badanidiyuru, Sebastien Bubeck, Chien-Ju Ho, Robert Kleinberg and Jennifer Wortman Vaughan for stimulating discussions on our problem and related research. Also, we thank Rajesh Patel, Steven Shelford and Hai Wu from Microsoft Bing for insights into the practical aspects of crowdsourcing. Finally, we are indebted to the anonymous referees for sharp comments which have substantially improved presentation. In particular, we thank anonymous reviewers for pointing out that our index-based algorithm can be interpreted via virtual rewards.


\begin{small}
\bibliographystyle{alpha}
\bibliography{bib-abbrv,bib-bandits,bib-slivkins,bib-crowdsourcing}
\end{small}

\appendix

\section{A missing proof from Section~\ref{sec:randomized-benchmark}}

In the proof of Lemma~\ref{lm:benchmarks-2options}, we have used the following general vector inequality:

\begin{claim}\label{cl:standard-inequality}
    $(\vec{x}\cdot \vec{\alpha})(\vec{x}\cdot \vec{\beta}) \geq \min_i\alpha_i \beta_i$
for any vectors $\vec{\alpha},\vec{\beta}\in \R^k_+$ and any $k$-dimensional distribution $\vec{x}$.
\end{claim}

This inequality appears standard, although we have not been able to find a reference. We supply is a self-contained proof below.

\begin{proof}
W.l.o.g. assume
    $\alpha_1 \beta_1 \leq \alpha_2 \beta_2 \leq \ldots \leq \alpha_k \beta_k$.
Let us use induction on $k$, as follows. Let
\[ f(\vec{x})
    \triangleq (\vec{x}\cdot \vec{\alpha})(\vec{x}\cdot \vec{\beta})
    = (x_1 \alpha_1 + A)(x_1\beta_1 +B) \]
where
\[    \begin{cases}
        A &= \sum_{i>1} x_i \alpha_i \\
        B &= \sum_{i>1} x_i \beta_i
    \end{cases}.\]
Denoting $p = x_1$, we can write the above expression as
\begin{align}\label{eq:standard-inequality-1}
    f(\vec{x}) = p^2 \alpha_1 \beta_1 + p(\alpha_1 B + \beta_1 A) + AB.
\end{align}

First, let us invoke the inductive hypothesis to handle the $AB$ term in~\eqref{eq:standard-inequality-1}. Let $y_i = \tfrac{x_i}{1-p}$ and note that
    $\{y_i\}_{i>1}$
is a distribution. It follows that
    $\tfrac{A}{1-p} \tfrac{B}{1-p} \geq \alpha_2 \beta_2$.
In particular,
    $AB\geq (1-p)^2 \alpha_1\beta_1$.

Next, let us handle the second summand in~\eqref{eq:standard-inequality-1}. Let us re-write it to make things clearer:
\begin{align}
\alpha_1\, B + \beta_1\, A
    &= (1-p)\; \sum_{i>1}\; \alpha_1\, y_i\, \beta_i + \beta_1\, y_i\, \alpha_i \nonumber \\
    &= (1-p)\,\alpha_1 \beta_1 \sum_{i>1}\; y_i
        \left(\frac{\alpha_i}{\alpha_1} + \frac{\beta_i}{\beta_1} \right). \label{eq:standard-inequality-2}
\end{align}
We handle the term in big brackets using the assumption that
    $\alpha_1 \beta_1 \leq \alpha_i \beta_i$.
By this assumption it follows that
    $\tfrac{\alpha_i}{\alpha_1} \geq \tfrac{\beta_1}{\beta_i}$
and therefore
    $\tfrac{\alpha_i}{\alpha_1} + \tfrac{\beta_i}{\beta_1} \geq
        \tfrac{\beta_1}{\beta_i} + \tfrac{\beta_i}{\beta_1} \geq 2$.
Plugging this into~\eqref{eq:standard-inequality-2}, we obtain
    \[ \alpha_1 B + \beta_1 A \geq 2(1-p)\, \alpha_1 \beta_1.\]

Finally, going back to~\eqref{eq:standard-inequality-1} we obtain
\begin{align*}
 f(\vec{x})
    &\geq p^2 \,\alpha_1\beta_1 + 2p(p-1) \,\alpha_1\beta_1 + (1-p)^2 \,\alpha_1\beta_1 \\
    &= \alpha_1\beta_1. \qedhere
 \end{align*}
\end{proof}

\OMIT{
\xhdr{Comparison of worst-case error rates.}
Fix a single-crowd stopping rule $R_0$. We would like to argue that the worst-case error rate of an arbitrary crowd-selection rule $\A$, when used with $R_0$, is not much smaller than the worst-case error rate of the benchmarks, i.e. that of $R_0$.

We will need a mild assumption on $\A$: essentially, that it never commits to stop using any given crowd. Formally, a crowd-selection algorithm $\A$ is called \emph{non-committing} if for every problem instance, each time $t$, and every crowd $i$, it will choose crowd $i$ at some time after $t$ with probability one. (Here we consider a run of $\A$ that continues indefinitely, without being stopped by the stopping rule.)

\begin{lemma}\label{lm:error-rate-LB}
Let $R_0$ be a symmetric single-crowd stopping rule with worst-case error rate $\rho$. Let $\A$ be a non-committing crowd-selection algorithm, and let $R$ be the composite stopping rule based on $R_0$ which does not use the total crowd. If $\A$ is used in conjunction with $R$, the worst-case error rate is at least $\rho\,(1-2k\rho)$, where $k$ is the number of crowds.
\end{lemma}

\begin{proof}
Suppose $R_0$ attains the worst-case error rate for a crowd with gap $\gap$. Consider the problem instance in which one crowd (say, crowd $1$) has gap $\gap$ and all other crowds have gap $0$. Let $R_{(i)}$ be the instance of $R_0$ that takes inputs from crowd $i$, for each $i$. Let $E$ be the event that each $R_{(i)}$, $i>1$ does not ever stop. Let $E'$ be the event that $R_{(1)}$ stops and makes a mistake. These two events are independent, so the error rate of $R$ is at least $\Pr[E]\, \Pr[E']$. By the choice of the problem instance, $\Pr[E']=\rho$. And by Lemma~\ref{lm:LB-infty}, $\Pr[E] \geq 1-2k\rho$. It follows that the error rate of $R$ is at least $\rho\, (1-2k\rho)$.
\end{proof}

}

\end{document}